%% file: main.tex
\documentclass[10pt]{article} % For LaTeX2e
\usepackage[accepted]{tmlr}
\input{math_commands.tex}

\title{On Sample Complexity of Offline Reinforcement Learning with Deep ReLU Networks in Besov Spaces}

\author{\name Thanh Nguyen-Tang  \email thnguyentang@gmail.com \\
      \addr Department of Computer Science, Johns Hopkins University 
      \AND
      \name  Sunil Gupta \& Hung Tran-The \& Svetha Venkatesh \\
      \addr Applied AI Institute, Deakin University 
      }

\def\month{12}  % Insert correct month for camera-ready version
\def\year{2022} % Insert correct year for camera-ready version
\def\openreview{\url{https://openreview.net/forum?id=LdEm0umNcv}} % Insert correct link to OpenReview for camera-ready version

\begin{document}

\maketitle

\begin{abstract}
\input{body/abstract}
\end{abstract}

\input{body/intro}

\input{body/related}
\input{body/background}
\input{body/algo} 
\input{body/theory}
\input{body/technical_review}
\input{body/conclusion}
\bibliography{main}
\bibliographystyle{tmlr}

\appendix
\section{Appendix}
\input{body/appendixA}

\input{body/appendixB}
\end{document}

%% file: math_commands.tex
\usepackage{amsmath,amsfonts,bm}
\usepackage[backref=page]{hyperref}
\usepackage{url}
\usepackage{amssymb,amsthm,mathtools}
\usepackage{wrapfig}
\usepackage{subfloat}
\usepackage{algorithm}
\usepackage{algpseudocode}
\usepackage{soul}
\usepackage{xcolor,cancel}

\newcommand{\subopt}{\mathrm{SubOpt}}

\def\eqref#1{equation~\ref{#1}}
\def\floor#1{\lfloor #1 \rfloor}
\def\1{\bm{1}}

\def\eps{{\epsilon}}

\DeclareMathAlphabet{\mathsfit}{\encodingdefault}{\sfdefault}{m}{sl}
\SetMathAlphabet{\mathsfit}{bold}{\encodingdefault}{\sfdefault}{bx}{n}

\def\gF{{\mathcal{F}}}

\def\sE{{\mathbb{E}}}

\def\sN{{\mathbb{N}}}

\def\sR{{\mathbb{R}}}

\def\nn{{\mathrm{NN}}}
\def\fnn{{\gF_{\nn}}}

\DeclareMathOperator*{\argmax}{arg\,max}
\DeclareMathOperator*{\argmin}{arg\,min}

\DeclareMathOperator*{\arginf}{arg\,inf}

\newtheorem{thm}{Theorem}[section]
\newtheorem{eg}{Example}[section]
\newtheorem{lem}{Lemma}[section]
\newtheorem{prop}{Proposition}[section]

\newtheorem{defn}{Definition}[section]

\newtheorem{assumption}{Assumption}[section]

\theoremstyle{remark}
\newtheorem{rem}{Remark}[section]

%% file: body/abstract.tex
Offline reinforcement learning (RL) leverages previously collected data for policy optimization without any further active exploration. Despite the recent interest in this problem, its theoretical results in neural network function approximation settings remain elusive. In this paper, we study the statistical theory of offline RL with deep ReLU network function approximation. In particular, we establish the sample complexity of $n = \tilde{\mathcal{O}}( H^{4 + 4 \frac{d}{\alpha}} \kappa_{\mu}^{1 + \frac{d}{\alpha}} \epsilon^{-2 - 2\frac{d}{\alpha}} )$ for offline RL with deep ReLU networks, where $\kappa_{\mu}$ is a measure of distributional shift, {$H = (1-\gamma)^{-1}$ is the effective horizon length},  $d$ is the dimension of the state-action space, $\alpha$ is a (possibly fractional) smoothness parameter of the underlying Markov decision process (MDP), and $\epsilon$ is a user-specified error. Notably, our sample complexity holds under two novel considerations: the Besov dynamic closure and the correlated structure. While the Besov dynamic closure subsumes the dynamic conditions for offline RL in the prior works, the correlated structure renders the prior works of offline RL with general/neural network function approximation improper or inefficient {in long (effective) horizon problems}. To the best of our knowledge, this is the first theoretical characterization of the sample complexity of offline RL with deep neural network function approximation under the general Besov regularity condition that goes beyond {the linearity regime} in the traditional Reproducing Hilbert kernel spaces and Neural Tangent Kernels. 

%% file: body/intro.tex
\section{Introduction}
Offline reinforcement learning \citep{lange2012batch,levine2020offline} is a practical paradigm of reinforcement learning (RL) where logged experiences are abundant but a new interaction with the environment is limited or even prohibited. The fundamental offline RL problems concern with how well previous experiences could be used to evaluate a new target policy, known as off-policy evaluation (OPE) problem, or to learn the optimal policy, known as off-policy learning (OPL) problem. We study these offline RL problems with infinitely large state spaces, where the agent must rely on function approximation such as deep neural networks to generalize across states from an offline dataset. Such problems form the core of modern RL in practical settings \citep{levine2020offline,kumar2020conservative,singh2020cog,zhang2022two}.

Prior sample-efficient results in offline RL mostly focus on tabular environments with small finite state spaces \citep{DBLP:conf/aistats/YinW20,DBLP:conf/aistats/YinBW21,yin2021characterizing}, but as these methods scale with the number of states, they are infeasible for the settings with infinitely large state spaces. While this tabular setting has been extended to large state spaces via {linear} environments \citep{DBLP:journals/corr/abs-2002-09516,jin2020pessimism,Xiong2022NearlyMO,yinnear,nguyen2022instance}, the linearity assumption often does not hold for many RL problems in practice. Theoretical guarantees for offline RL with general and deep neural network function approximations have also been derived, but these results are either inadequate or relatively disconnected from the regularity structure of the underlying MDP. In particular, while the finite-sample results for offline RL with general function approximation \citep{DBLP:journals/jmlr/MunosS08,DBLP:conf/icml/0002VY19} depend on an inherent Bellman error which could be uncontrollable in practice, the other analysis in the neural network function approximation in \cite{DBLP:journals/corr/abs-1901-00137} relies on a data splitting technique to deal with the correlated structures arisen in value regression for offline RL and use a relatively strong dynamic assumption. Recent works have studied offline RL with function approximations in Reproducing Hilbert kernel spaces (RHKS) \citep{jin2020pessimism} and Neural Tangent Kernels (NTK) \citep{nguyen2021offline}. However, these function classes also have (approximately) linear structures (in terms of the underlying features) which make their analysis similar to the linear case. Moreover, the smoothness assumption imposed by the RKHS is often strong for several practical cases while the NTK analysis requires a extremely wide neural net (the network width scales with $n^{10}$ for the NTK case in \cite{nguyen2021offline} versus only $n^{2/5}$ (Proposition \ref{prop: simplied version of the main theorem}) in the current work). Recent works \citep{xie2021bellman,zhan2022offline,Chen2022OfflineRL,uehara2021pessimistic} have considered offline RL with general function approximation with weaker data coverage assumption. However, they assumed the function class is finite and did not consider the (Besov) regularity of the underlying MDP. Thus, to our knowledge, no prior work has dedicated to study a comprehensive and adequate analysis of the statistical efficiency for offline RL with neural network function approximation {in Besov spaces}. Thus, it is natural to ask:
\begin{center}
\textit{Is offline RL sample-efficient with deep ReLU network function approximation beyond the (approximate-) linear regime imposed by RKHS and NTK?}
\end{center}

\paragraph{Our contributions.}
In this paper, we provide a statistical theory of both OPE and OPL with neural network function approximation in a broad generality that is beyond the (approximate-) linear regime imposed by RKHS and NTK. In particular, our contributions, which are summarized in Table \ref{tab:compare_literature} and will be discussed in details in Section \ref{section:main_results}, are:

% for the dynamic condition and the correlated structures. 
% In this paper, we study a variation of fitted-Q iteration (FQI) \citep{bertsekas1995dynamic,introRL} for the offline RL problems where we approximate the target $Q$-function from an offline data using a deep ReLU network. The algorithm is appealingly simple: it iteratively estimates the target $Q$-function via regression on the offline data and the previous estimate. This procedure, which intuitively does the best it could with the available offline data, forms the core of many current offline RL methods. We provide the statistical theory of a FQI-type algorithm for both OPE and OPL problems with deep ReLU networks. Our statistical theory is the first to provide a comprehensive analysis for offline RL under deep ReLU network function approximation: 
\begin{itemize}
    \item {First}, we achieve a generality for the guarantees of offline RL with neural network function approximation via two novel considerations: (i) We introduce a new structural condition namely 
    \emph{Besov dynamic closure} which subsumes the existing dynamic conditions for offline RL with neural network function approximation and even includes MDPs that need not be continuous, differentiable or spatially homogeneous in smoothness; (ii) We take into account the correlated structure of the value estimate produced by a regression-based algorithm from the offline data. This correlated structure plays a central role in the statistical efficiency of an offline algorithm; yet the prior works improperly ignore this structure \citep{DBLP:conf/icml/0002VY19} or avoid it using an data splitting approach \citep{DBLP:journals/corr/abs-1901-00137}.
    % , respectively. 
    
    % \end{itemize}
    
    \item {Second}, we prove that an offline RL algorithm based on fitted-Q iteration (FQI) can achieve the sample complexity of $n = \tilde{\mathcal{O}}( H^{4 + 4 \frac{d}{\alpha}} \kappa_{\mu}^{1 + \frac{d}{\alpha}} \epsilon^{-2 - 2\frac{d}{\alpha}} )$ where $\kappa$ measures the distributional shift in the offline data, $H = (1 - \gamma)^{-1}$ is the effective horizon length, $d$ is the input dimension, $\alpha$ is a smoothness parameter of the underlying MDP, and $\epsilon$ is a user-specified error. Notably, our guarantee holds under our two novel considerations above that generalize the condition in \cite{DBLP:journals/corr/abs-1901-00137} and do not require data splitting in \cite{DBLP:journals/corr/abs-1901-00137}. Moreover, our analysis also corrects the technical mistake in \cite{DBLP:conf/icml/0002VY19} that ignores the correlated structure of offline value estimate.   
    % holds under a general condition encompassing the dynamic conditions in the existing works
    % while it does not require any data splitting  as in \citep{DBLP:journals/corr/abs-1901-00137}. 
    % The data splitting approach splits the offline data into $K$ disjoint folds where $K$ is the number of iterations in their algorithm. As the sample complexity of such data splitting scales linearly with $K$ where $K$ can be arbitrarily large in practice, the guarantee in  \citep{DBLP:journals/corr/abs-1901-00137} is highly inefficient for offline RL. Moreover, our analysis also improves upon the analysis in \citep{DBLP:conf/icml/0002VY19} that incorrectly ignores the correlated structure of offline value estimate.   
\end{itemize}

\begin{table*}
    \centering
    \caption{Comparison among existing representative works of FQI estimators for offline RL with function approximation under a \emph{uniform} data coverage assumption.
    % The representative statistical theory of offline RL with function approximation under a uniform data coverage assumption. Here, the distributional shift measure $\kappa$ can be defined differently in different works.}
    Here $S$ and $A$ are the cardinalities of the state and action space when they are finite, $\kappa$ is a measure of distribution shift (which can be defined slightly different in different works), $\epsilon$ is the user-specified precision, $d$ is the dimension of the input space, $\alpha$ is the smoothness parameter of the underlying MDP, {and $H := (1-\gamma)^{-1}$ is the effective horizon length.}
    % \st{$K$ is the algorithmic iteration number}
    }
    % \thanh{If the authors claim the data-splitting will lead to severe issues, the authors should explicitly claim the dependency on $K$. If I understand correctly, we only require  to make the algorithmic error at a order of $K = \mathcal{O}(\log(1/\epsilon))$ , which, for me is not an issue in both theory and practice.}}
    \hspace{3cm}
    \def\arraystretch{1.8}%
     \resizebox{\textwidth}{!}{  
    \begin{tabular}{cccccc}
    \hline
       \textbf{Work} &  \textbf{Function} & \textbf{Regularity} & \textbf{Tasks} & \textbf{Sample complexity}  & \textbf{Remark}\\
    \hline
    \hline 
    
    \citet{DBLP:conf/aistats/YinW20}  & Tabular & Tabular & OPE & $\tilde{\mathcal{O}}\left( \kappa \cdot H^4 \cdot \epsilon^{-2} \cdot (S A)^2   \right)$  & - \\ 
    \hline 
    
    \citet{DBLP:journals/corr/abs-2002-09516} & Linear & Linear & OPE & $\tilde{\mathcal{O}}\left( \kappa \cdot H^4 \cdot \epsilon^{-2} \cdot d   \right)$  & -\\ 
    \hline 
    \citet{DBLP:conf/icml/0002VY19} & General & General & OPE/OPL & N/A  & improper analysis\\ 
    \hline 
    
    \citet{DBLP:journals/corr/abs-1901-00137} & ReLU nets & H\"older & OPL & $ \tilde{\mathcal{O}}\left( \kappa^{2 + \frac{d}{\alpha}} \cdot H^{5 + 2\frac{d}{\alpha}} \cdot \epsilon^{-2 - \frac{d}{\alpha}  }  \cdot \log(H^2/\epsilon) \right)$  & data splitting  \\
    \hline 
    
    This work &  {ReLU nets} & {Besov}  & {OPE/OPL} & $\tilde{\mathcal{O}}\left( \kappa^{1 + \frac{d}{\alpha}} \cdot H^{4 + 4\frac{d}{\alpha}}  \cdot \epsilon^{-2 - 2\frac{d}{\alpha}} \right)$ & {data reuse} \\
    \hline 
    
    \end{tabular}
    }
    \label{tab:compare_literature}
\end{table*}
\paragraph{Problem scope.} We emphasize that the present work focuses on statistical theory of offline RL with neural network function approximation in Besov spaces where we analyze a relatively standard algorithm, FQI. Regarding the empirical effectiveness of FQI with neural network function approximation for offline RL, we refer the readers to the empirical study in \cite{voloshin2019empirical}. Finally, this work is an extension of our workshop paper \citep{nguyentang2021sample}. 

% \textbf{Notation}. Let $L^p(\mathcal{X}, \mu) = \{f: \mathcal{X} \rightarrow \mathbb{R} \text{ }|\text{ } \|f\|_{p, \mu} := (\int_{\mathcal{X}} |f|^p d\mu)^{1/p} < \infty \}$ be the space of measurable functions for which the $p$-th power of the absolute value is $\mu$-measurable, $C^0(\mathcal{X}) = \{f: \mathcal{X} \rightarrow \mathbb{R} \text{ }|\text{ } f \text{ is continuous and } \|f\|_{\infty}  < \infty \}$ be the space of bounded continuous functions, $C^{\alpha}(\mathcal{X})$ be the H\"older space with smoothness parameter $\alpha \in (0, \infty) \backslash \mathbb{N}$,  $W^{m}_p(\mathcal{X})$ be the Sobolev space with regularity $m \in \mathbb{N}$ and parameter $p \in [1,\infty]$, and $X \hookrightarrow Y$ be \textit{continuous embedding} from a metric space $X$ to a metric space $Y$. Denote by $\mathcal{P}(\Omega)$ the set of probability measures supported in domain $\Omega$. 
\paragraph{Notations.} 
Denote $\|f\|_{p, \mu} := \displaystyle \left(\int_{\mathcal{X}} |f|^p d\mu \right)^{1/p}$, and for simplicity, we write $\|\cdot\|_{\mu}$ for $\| \cdot\|_{p,\mu}$ when $p=2$ {and write $\| \cdot \|_p$ for $\| \cdot\|_{p,\mu}$ if $\mu$ is the Lebesgue measure}. {Let $L^p(\mathcal{X})$ be the space of measurable functions for which the $p$-th power of the absolute value is Lebesgue integrable, i.e. $L^p(\mathcal{X}) = \{f: \mathcal{X} \rightarrow \mathbb{R} | \|f\|_{p} < \infty\}$.} Denote by $\|\cdot\|_0$ the $0$-norm, i.e., the number of non-zero elements, and $a \lor b = \max\{a,b\}$. For any two real-valued functions $f$ and $g$, we write $f(\cdot) \lesssim g(\cdot)$ if there is an absolute constant $c$ independent of the function parameters $(\cdot)$ such that $f(\cdot) \leq c \cdot g(\cdot)$. We write $f(\cdot) \asymp g(\cdot)$ if $f(\cdot) \lesssim g(\cdot)$ and $g(\cdot) \lesssim f(\cdot)$. We write $f(\cdot) \simeq g(\cdot)$ if there exists an absolute constant $c$ such that $f(\cdot) = c \cdot g(\cdot)$. {We denote $H := (1-\gamma)^{-1}$ which is the effective horizon length in the discounted MDP and is equivalent to the horizon (episode) length in finite-horizon MDPs.}

%% file: body/related.tex
\section{Related Work}
% This section presents the necessary background of OPE and the related literature for our work. 
\paragraph{Offline RL with tabular representation.} The majority of the theoretical results for offline RL focus on tabular MDP where the state space is finite and an importance sampling -related approach is possible \citep{first_is,dudik2011doubly,jiang2015doubly,thomas2016data,farajtabar2018more,kallus2019double}. The main drawback of the importance sampling-based approach is that it suffers high variance in long horizon problems. The high variance problem can be mitigated by direct methods where we employ models to estimate the value functions or the transition kernels. We focus on direct methods in this work. For tabular MDPs with some uniform data-visitation measure $d_m > 0$, a near-optimal sample complexity bound of $\mathcal{O}(H^3 d_m / \epsilon^2)$ and $\mathcal{O}(H^2 d_m / \epsilon^2)$ were obtained for time-inhomogeneous tabular MDP \citep{DBLP:conf/aistats/YinBW21} and for time-homogeneous tabular MDP \citep{yin2021optimal,ren2021nearly}, respectively. With the single-concentrability assumption, a tight bound of $\mathcal{O}(H^3 S C^* / \epsilon^2)$ was achieved \citep{xie2021policy,rashidinejad2021bridging}, where $H \approx 1/(1-\gamma)$ is the episode length. \citet{yin2021towards} introduced intrinsic offline bound that further incorporates instance-dependent quantities. \citet{shi2022pessimistic} obtained the minimax rate with model-free methods. \citet{wang2022gap} derived gap-dependent bounds for offline RL in tabular MDPs. 

\paragraph{Offline RL with linear function approximation.} Offline RL with function approximation often follow two algorithmic approaches: Fitted Q-iteration (FQI) \citep{bertsekas1995neuro,DBLP:conf/atal/JongS07,DBLP:journals/jmlr/LagoudakisP03,DBLP:conf/icml/GrunewalderLBPG12,DBLP:conf/icml/Munos03,DBLP:journals/jmlr/MunosS08,10.1007/s10994-007-5038-2,DBLP:conf/icml/TosattoPDR17,DBLP:conf/icml/0002VY19}, and pessimism principle \citep{buckman2020importance}, {where the former requires a uniform data coverage and the latter only needs a sufficient coverage over the target policy}. 
% In tabular settings, 
% % The high variance problem is later mitigated by the idea of formulating the offline problem as a density ratio estimation problem  \citep{DBLP:conf/nips/LiuLTZ18,nachum2019dualdice,Zhang2020GenDICEGO,Zhang2020GradientDICERG,Nachum2019AlgaeDICEPG} but these results do not provide sample complexity guarantees. 
% the sample-efficient guarantees for offline RL are obtained in \citep{xie2019optimal,DBLP:conf/aistats/YinW20,DBLP:conf/aistats/YinBW21,yin2021characterizing}. \citet{DBLP:conf/icml/JiangL16} derive Cramer-Rao lower bound for discrete-tree MDPs. 
% Some theoretical progress for offline RL in linear models have been made \citep{DBLP:journals/corr/abs-2002-09516,wang2020statistical,jin2020pessimism,chen2021infinite}. 
\citet{DBLP:journals/corr/abs-2002-09516} studied fitted-Q iteration algorithm in linear MDPs. \citet{wang2020statistical} highlighted the necessity of strong structural assumptions (e.g., on low distributional shift or strong dynamic condition beyond realizability) for sample-efficient offline RL with linear function approximation suggesting that only realizability and strong uniform data coverage are not sufficient for sample-efficient offline RL. \citet{jin2020pessimism} brought pessimism principle into offline linear MDPs. \citet{nguyen2022instance} derived a minimax rate of $1/\sqrt{n}$ for offline linear MDPs under a partial data coverage assumption and obtained the instance-dependent rate of $1/n$ when the gap information is available. 
\citet{Xiong2022NearlyMO,yinnear} used variance reduction and data splitting to tighten the bound of \citet{jin2020pessimism}. \citet{xie2021bellman} proposed Bellman-consistent condition with general function approximation which improves the bound of \citet{jin2020pessimism} by a factor of $\sqrt{d}$ when realized to finite action space and linear MDPs.  
\citet{chen2021infinite} studied sample complexity of FQI in linear MDPs and derive a lower bound for this setting.
% \citet{wang2020statistical} study the statistical hardness of offline RL with linear representation 

\paragraph{Offline RL with non-linear function approximation.}
Beyond linearity, some works study offline RL in general or nonparametric function approximation, either with FQI estimators \citep{DBLP:journals/jmlr/MunosS08,DBLP:conf/icml/0002VY19,duan2021risk,duan2021optimal,hu2021fast}, pessimistic estimators \citep{uehara2021pessimistic,nguyen2021offline,jin2020pessimism}, or minimax estimators \citep{uehara2021finite}, where \citet{uehara2021finite} also realized their minimax estimators to the neural network function approximation, \citet{nguyen2021offline} considered offline contextual bandits with Neural Tangent Kernels (NTK), and \citet{jin2020pessimism} considered the pessimistic value iteration algorithm with Reproducing Kernel Hilbert Space (RKHS) in their extended version. Our work is different from these aforementioned works in that we analyze the fundamental FQI estimators with neural network function approximation under the Besov regularity condition that is much more general than RKHS and NTK. We also further emphasize that even that RKHS and NTK spaces are non-linear function approximation, the functions in those spaces are linear in terms of an underlying feature space, making the analysis for these spaces akin to the case of linear function approximation. \citet{DBLP:journals/corr/abs-1901-00137} also considered deep neural network approximation. In particular, \citet{DBLP:journals/corr/abs-1901-00137} focused on analyzing deep Q-learning using a disjoint fold of offline data for each iteration. Such approach is considerably sample-inefficient for offline RL {with long (effective) horizon}. 
% \st{as their sample complexity linearly scales with the number of iterations} $K$ \st{which is very large in practice}. 
In addition, they rely on a relatively restricted smoothness assumption of the underlying MDPs that hinders their results from being widely applicable in more general settings. Recently, other works \citep{xie2021bellman,zhan2022offline,Chen2022OfflineRL,uehara2021pessimistic} considered offline RL with general function approximation and imposed weaker data coverage assumption by using pessimistic algorithms. Their algorithms are more involved than FQI but did not study the effect of the regularity of the underlying MDP on the sample complexity of offline RL. They also assume that the function class is finite which is not applicable to neural network function approximation. Since the first version of our paper appeared online, there have been several other works establishing sample complexity of reinforcement learning in Besov spaces for various problem settings, including $\epsilon$-greedy exploration for online setting with Markovian data \citep{liu2022understanding} and off-policy evaluation on low-dimensional manifolds  \citep{ji2022sample}.

%% file: body/background.tex
\section{Preliminaries}
\label{section:preliminaries}
% \thanh{I'M HERE}
We consider a discounted Markov decision process (MDP) with possibly infinitely large state space $\mathcal{S}$, continuous action space $\mathcal{A}$, initial state distribution $\rho \in \mathcal{P}(\mathcal{S})$, transition operator $P: \mathcal{S} \times \mathcal{A} \rightarrow \mathcal{P}(\mathcal{S})$, reward distribution $R: \mathcal{S} \times \mathcal{A} \rightarrow \mathcal{P}([0,1])$, and a discount factor $\gamma \in [0,1)$. For notational simplicity, we assume that $\mathcal{X} := \mathcal{S} \times \mathcal{A} \subseteq [0,1]^d$. 
% but our results readily generalizes to the case when $\mathcal{A}$ is finite. 

A policy $\pi: \mathcal{S} \rightarrow \mathcal{P}(\mathcal{A})$ induces a distribution over the action space conditioned on states. The $Q$-value function for policy $\pi$ at state-action pair $(s,a)$, denoted by $Q^{\pi}(s,a) \in [0,1]$, is the expected discounted total reward the policy collects if it initially starts in the state-action pair, i.e., 
\begin{align*}
    Q^{\pi}(s,a) &:= \mathbb{E}_{\pi} \left[ \sum_{t=0}^{\infty} \gamma^t r_t | s_0 = s, a_0 = a \right], 
\end{align*}
where $r_t \sim R(s_t, a_t), a_t \sim \pi(\cdot|s_t)$, and $s_t \sim P(\cdot|s_{t-1}, a_{t-1})$. The value of policy $\pi$ is
% \begin{align*}
    $V^{\pi} = \mathbb{E}_{s \sim \rho, a \sim \pi(\cdot|s)} \left[Q^{\pi}(s,a) \right]$,
% \end{align*}
and the optimal value is $V^* = \max_{\pi} V^{\pi}$ where the maximization is taken over all stationary policies. Alternatively, the optimal value $V^*$ can be obtained via the optimal $Q$-function $Q^* = \max_{\pi} Q^{\pi}$ as $V^* = \mathbb{E}_{s \sim \rho} \left[ \max_{a} Q^*(s,a) \right]$. Denote by $T^{\pi}$ and $T^*$ the Bellman operator and the optimality Bellman operator, respectively, i.e., for any $f: \mathcal{S}\times \mathcal{A} \rightarrow \mathbb{R}$
\begin{align*}
    [T^{\pi} f](s,a) &= \mathbb{E}_{r \sim R(s,a)}[r] + \gamma \mathbb{E}_{s' \sim P(\cdot|s,a), a' \sim \pi(\cdot|s')} \left[ f(s',a') \right] \\
    [T^* f](s,a) &= \mathbb{E}_{r \sim R(s,a)}[r] + \gamma \mathbb{E}_{s' \sim P(\cdot|s,a)} \left[ \max_{a'} f(s',a') \right], 
\end{align*}
we have $T^{\pi} Q^{\pi} = Q^{\pi}$ and $T^* Q^* = Q^*$.

\paragraph{Offline regime.} We consider the offline RL setting where a learner cannot explore the environment but has access to a fixed logged data $\mathcal{D} = \{(s_i, a_i, s'_i, r_i)\}_{i=1}^n$ collected a priori by certain behaviour policy $\eta$. For simplicity, we assume that $\{s_i\}_{i=1}^n$ are independent and $\eta$ is stationary. Equivalently, $\{(s_i, a_i)\}_{i=1}^n$ are i.i.d. samples from the normalized discounted stationary distribution over state-actions with respect to $\eta$, i.e., $(s_i, a_i) \overset{i.i.d.}{\sim} \mu(\cdot, \cdot) := (1 - \gamma) \sum_{t=0}^{\infty} \gamma^t \mathbb{P}(s_t = \cdot, a_t = \cdot| \rho, \eta )$ where $s'_i \sim P(\cdot| s_i, a_i)$ and $a_i \sim \eta(\cdot | s_i)$. This assumption is relatively standard in the offline RL setting \citep{DBLP:journals/jmlr/MunosS08,DBLP:conf/icml/ChenJ19,DBLP:journals/corr/abs-1901-00137}.
% and is used merely for the sake of theoretical analysis. 
% \st{where $(s_i, a_i) \overset{i.i.d.}{\sim} \mu(\cdot, \cdot) := \frac{1}{1 - \gamma} \sum_{t=0}^{\infty} \gamma^t P(s_t = \cdot, a_t = \cdot| \rho, \eta), s'_i \sim P(\cdot| s_i, a_i)$ and $r_i \sim R(s_i, a_i)$}. \st{Here $\mu$ is the (sampling) state-action visitation distribution}. 
The goals of Off-Policy Evaluation (OPE) and Off-Policy Learning (OPL) are to estimate $V^{\pi}$ and $V^*$, respectively from $\mathcal{D}$. The performance of OPE and OPL estimates are measured via sub-optimality gaps defined as follows. 
\paragraph{For OPE Task.} Given a fixed target policy $\pi$, for any value estimate $\hat{V}$ computed from the offline data $\mathcal{D}$, the sub-optimality of OPE is defined as  
\begin{align*}
    \text{SubOpt}(\hat{V}; \pi) = |V^{\pi} - \hat{V}|. 
\end{align*}

\paragraph{For OPL Task.} For any estimate $\hat{\pi}$ of the optimal policy $\pi^*$ that is learned from the offline data $\mathcal{D}$, we define the sup-optimality of OPL as 
% \begin{align*}
%   \hcancel{ \text{SubOpt}(\hat{\pi}) = \mathbb{E}_{\rho} \left[ V^*(s) - Q^*(s, \hat{\pi}(s)) \right]}, 
% \end{align*}
\begin{align*}
    {\subopt(\hat{\pi}) := \sE_{s \sim \rho} \left[ V^*(s) - V^{\hat{\pi}}(s) \right]}.
\end{align*}

\subsection{Deep ReLU Networks as Function Approximation}
\label{subsection:deep_relu_net}
In practice, the state space is often very large and complex, and thus function approximation is required to ensure generalization across different states. Deep neural networks with the ReLU activation offer a rich class of parameterized functions with differentiable parameters. Deep ReLU networks are state-of-the-art in many applications, e.g., \cite{krizhevsky2012imagenet,mnih2015human}, including offline RL with deep ReLU networks that can yield superior empirical performance \citep{voloshin2019empirical}. In this section, we describe the architecture of deep ReLU networks and the associated function space which we use throughout this paper. Specifically, a $L$-height, $m$-width ReLU network on $\mathbb{R}^d$ takes the form of 
\begin{align*}
    f_{\theta}^{L,m}(x) = W^{(L)} \sigma \left(  W^{(L-1)} \sigma \left( \hdots \sigma \left(W^{(1)} \sigma(x) + b^{(1)} \right) \hdots \right) + b^{(L-1)} \right) + b^{(L)},
\end{align*}
where $W^{(L)} \in \mathbb{R}^{1 \times m}, b^{(L)} \in \mathbb{R}, W^{(1)} \in \mathbb{R}^{m \times d}, b^{(1)} \in \mathbb{R}^m$, $W^{(l)} \in \mathbb{R}^{m \times m}, b^{(l)} \in \mathbb{R}^m, \forall 1 < l < L$, $\theta = \{W^{(l)}, b^{(l)}\}_{1 \leq l \leq L}$, and $\sigma(x) = \max\{x, 0\}$ is the (element-wise) ReLU activation. We define $\Phi(L, m, S,B)$ as the space of $L$-height, $m$-width ReLU functions $f_{\theta}^{L,m}(x)$ with sparsity constraint $S$, and norm constraint $B$, i.e., $\sum_{l=1}^L (\|W^{(l)}\|_0 + \| b^{(l)} \|_0) \leq S, \max_{1 \leq l \leq L} \|W^{(l)} \|_{\infty} \lor \|b^{(l)} \|_{\infty} \leq B$. Finally, for some $L,m \in \mathbb{N}$ and $S,B \in (0,\infty)$, we define the unit ball of ReLU network function space $\mathcal{F}_{NN}$ as 
\begin{align*}
    {\fnn(L,m,S,B)} := \bigg\{ f \in \Phi(L,m,S,B):  \| f \|_{\infty} \leq 1 \bigg\}. 
\end{align*}
% We further write $\mathcal{F}_{NN}(\mathcal{X})$ to emphasize the domain $\mathcal{X}$ of deep ReLU functions in $\mathcal{F}_{NN}$ but often use $\mathcal{F}_{NN}$ when the domain context is clear. 
In nonparametric regressions, \citet{suzuki2018adaptivity} showed that deep ReLU networks outperform any non-adaptive linear estimator due to their higher adaptivity to spatial inhomogeneity.

\subsection{Besov spaces}
Our new dynamic condition relies on the regularity of Besov spaces. 
% In this section, we define a function space for the target functions for which we study offline RL. Note that a regularity assumption on the target function is necessary to obtain a nontrivial rate of convergence \citep{DBLP:books/daglib/0035701}. A common way to measure regularity of a function is through the $L^p$-norm of its local oscillations (e.g., of its derivatives if they exist). This regularity notion encompasses the classical Lipschitz, H\"older and Sobolev spaces. In particular in this work, we consider Besov spaces. 
% % Unlike the previous spaces considered in offline RL such as H\"older and Sobolev spaces that permit only integer smoothness, 
% Besov spaces allow \textit{fractional} smoothness that describes the regularity of a function more precisely and generalizes the previous smoothness notions.   
% a function class with allows fractional smoothness that can describe 
% general smoothness that further generalizes the smoothness conditions considered in the offline RL literature. 
There are several ways to characterize the smoothness in Besov spaces. Here, we pursue a characterization via {multivariate} moduli of smoothness as it is more intuitive, following \cite{gine2016mathematical}. 

\begin{defn}[\textit{{Multivariate} moduli of smoothness}]
% \blue{For any function $f \in L^p(\mathcal{X})$, any $h = (h_1, \ldots, h_d) \in \mathbb{R}^d$, and $r \in \mathbb{N}$, the $r$-th order translation-difference operator $\Delta_h^r$ is defined as}
% \begin{align*}
%     \Delta_h^r(f)(\cdot) := \sum_{k=0}^r {{r}\choose{k}} (-1)^{r-k} f(\cdot + k\cdot h). 
% \end{align*}
For any $t > 0$
% \blue{$t = (t_1, \ldots, t_d) \in \mathbb{R}^d_{+} := (0, \infty)^d$} 
and $r \in \mathbb{N}$, the $r$-th {multivariate} modulus of smoothness of any function $f \in L^p(\mathcal{X}), p \in [1, \infty]$ is defined as
\begin{align*}
    \omega_r^{t,p}(f) := \sup_{0 \leq {\|h\|} \leq t} \| \Delta_h^r(f) \|_p,
\end{align*}
where 
% \blue{$h = (h_1, \ldots, h_d) \in \mathbb{R}^d$, $0 \leq h \leq t$ denotes $\{0 \leq h_i \leq t_i, \forall i \in [d]\}$}, and
$\Delta_h^r(f)$ is the $r$-th order translation-difference operator defined as 
\begin{align*}
    \Delta_h^r(f)(\cdot) := \sum_{k=0}^r {{r}\choose{k}} (-1)^{r-k} f(\cdot + k\cdot h). 
\end{align*}
\end{defn}

% For a function $f \in L^p(\mathcal{X})$ for some $p \in [1, \infty]$, we define its $r$-th \blue{multivariate} modulus of smoothness as $\omega_r^{t,p}(f) := \sup_{0 \leq h \leq t} \| \Delta_h^r(f) \|_p$, \blue{for any $t \in \mathbb{R}^d_{+} := (0, \infty)^d$, $r \in \mathbb{N}$ where $0 \leq h \leq t$ denotes that $0 \leq h_i \leq t_i, \forall i \in [d]$}

% $t > 0, r \in \mathbb{N}$,
% % \begin{align*}
% %     \omega_r^{t,p}(f) := \sup_{0 \leq h \leq t} \| \Delta_h^r(f) \|_p, t > 0, r \in \mathbb{N}, 
% % \end{align*}
% where the $r$-th order translation-difference operator $\Delta_h^r = \Delta_h \circ \Delta_h^{r-1}$ is recursively defined as 
% \begin{align*}
%     \Delta_h^r(f)(\cdot) &:= (f(\cdot + h) - f(\cdot))^r = \sum_{k=0}^r {{r}\choose{k}} (-1)^{r-k} f(\cdot + k\cdot h). 
% \end{align*}
% \end{defn}
% \thanh{What is dimension of $\mathcal{X}$? By definition 3.1, it looks like $\mathcal{X} \subset \mathbb{R}$?}
\begin{rem}
The quantity $\Delta_h^r(f)$ captures the local oscillation of $f$ which is not necessarily differentiable. In the case the $r$-th order weak derivative $D^r f$ exists and is locally integrable, we have 
\begin{align*}
    \lim_{ h \rightarrow 0} \frac{\Delta^r_h(f)(x)}{{\| h\|^r}} = D^r f(x). 
    % \frac{\omega^{t,p}_r(f)}{t^r} \leq \| D^r f \|_p \text{ and } \frac{\omega^{t,p}_{r+r'}(f)}{t^r} \leq \omega^{t,p}_{r'}(D^r f). 
\end{align*}
% \thanh{The statements in Remark 3.1 are repeating}
It also follows from Minkowski's inequality that 
\begin{align*}
    \frac{\omega^{t,p}_r(f)}{ t^r} \leq \| D^r f \|_p \text{ and } \frac{\omega^{t,p}_{r+r'}(f)}{t ^r} \leq \omega^{t,p}_{r'}(D^r f). 
\end{align*}
\end{rem}

\begin{defn}[Besov space $B^{\alpha}_{p,q}(\mathcal{X})$]
For $1 \leq p,q \leq \infty$ and $\alpha > 0$, we define the norm $\|\cdot\|_{B^{\alpha}_{p,q}}$  of the Besov space $B^{\alpha}_{p,q}(\mathcal{X})$ as $\|f\|_{B^{\alpha}_{p,q}} := \|f\|_p + |f |_{B^{\alpha}_{p,q}}$ where 
\begin{align*}
    |f |_{B^{\alpha}_{p,q}} := 
    \begin{cases}
    \left( \displaystyle \int_{\mathbb{R}_+} \left(\frac{\omega_{\floor{\alpha} + 1}^{t,p}(f)}{t ^{\alpha}} \right)^q \frac{dt}{  t } \right)^{1/q}, & 1 \leq q < \infty, \\ 
    \sup_{t > 0} \displaystyle \frac{\omega_{\floor{\alpha} + 1}^{t,p}(f)}{  t ^{\alpha} }, & q = \infty, 
    \end{cases}
\end{align*}
is the Besov seminorm. Then, $B^{\alpha}_{p,q} := \{ f \in L^p(\mathcal{X}) : \|f \|_{B^{\alpha}_{p,q}} < \infty \}$. 
\end{defn}

Intuitively, the Besov seminorm $|f |_{B^{\alpha}_{p,q}}$ roughly describes the $L^q$-norm of the $l^p$-norm of the $\alpha$-order smoothness of $f$. Besov spaces are considerably general that subsume H\"older spaces and Sobolev spaces as well as functions with spatially inhomogeneous smoothness \citep{Triebel1983TheoryOF,besovbook,suzuki2018adaptivity,primer_besov,Nickl2007BracketingME}. In particular, the Besov space $B^{\alpha}_{p,q}$ reduces into the H\"older space $C^{\alpha}$ when $p=q=\infty$ and $\alpha$ is  positive and non-integer while it reduces into the Sobolev space $W^{\alpha}_2$ when $p=q=2$ and $\alpha$ is a positive integer. We further consider the unit ball of $B^{\alpha}_{p,q}(\mathcal{X})$ as $\bar{B}^{\alpha}_{p,q}(\mathcal{X}) := \{g \in B^{\alpha}_{p,q}: \|g \|_{B^{\alpha}_{p,q}} \leq 1 \text{ and } \|g\|_{\infty} \leq 1\}$. When the context is clear, we drop $\mathcal{X}$ from $\bar{B}^{\alpha}_{p,q}(\mathcal{X})$.

%% file: body/algo.tex
\section{Fitted Q-Iteration for Offline Reinforcement Learning}
% \begin{wrapfigure}{R}{0.55\textwidth}
% \begin{minipage}{0.55\textwidth}
 \begin{algorithm}[H]
      \begin{algorithmic}[1]
        % \SetAlgoLined
        \caption{Fitted Q-Iteration with Neural Network Function Approximation}
        \label{alg:LSVI}
        % \KwData{$n \geq 0$}
        % \KwResult{$y = x^n$}
        \State \textbf{Input: } Offline data  $\mathcal{D} = \{(s_i, a_i, s'_i, r_i)\}_{i=1}^n$, number of iterations $K$, function family $\gF_{\mathrm{NN}}$, target policy $\pi$ (for OPE Task only)
        \State Initialize $Q_0 \in \gF_{\mathrm{NN}}$\;
        
        \For{$k=1, \ldots, K$ }
            % \State If \textbf{OPE Task}:  $y_i \leftarrow r_i + \gamma \int_{\mathcal{A}} Q_{k-1}(s'_i, a) \pi(da|s'_i), \forall i$ \; 
            
            % \State If \textbf{OPL Task}:  $y_i \leftarrow r_i + \gamma \max_{a' \in \mathcal{A}} Q_{k-1}(s'_i, a'), \forall i $ \;
            
            % \State $Q_k \leftarrow \argmin_{f \in \mathcal{F}_{NN}} \frac{1}{n} \sum_{i=1}^n (f(s_i, a_i) - y_i)^2 $ \;
            % \State For \textbf{OPE Task}, 
            \State Compute the estimated state-action value $Q_k$ as
            \begin{align*}
            \begin{cases}
                 Q_k \leftarrow \argmin_{f \in \fnn} \frac{1}{n} \sum_{i=1}^n \left(f(s_i, a_i) - r_i - \gamma \mathbb{E}_{a' \sim \pi(\cdot|s'_i)} \left[ Q_{k-1}(s'_i, a) \right] \right)^2 & \text{ if OPE Task} \\
                  Q_k \leftarrow \argmin_{f \in \fnn} \frac{1}{n} \sum_{i=1}^n \left(f(s_i, a_i) - r_i - \gamma \max_{a \in \mathcal{A}} Q_{k-1}(s'_i, a) \right)^2 & \text{ if OPL Task}
            \end{cases}
            \end{align*}
            % \State For \textbf{OPL Task}, 
            % \begin{align*}
               
            % \end{align*}
        \EndFor
        \State \textbf{Output:} Return the following estimates
        \begin{align*}
            \begin{cases}
                     V_K \leftarrow \|Q_K \|_{ \rho^{\pi}} := \displaystyle \sqrt{ \mathbb{E}_{\rho(s) \pi(a|s)} \left[ Q_K(s,a)^2 \right]} & \text{ if OPE Task} \\ 
                     \pi_K(\cdot|s) \leftarrow \argmax_{a} Q_K(a|s) & \text{ if OPL Task}
            \end{cases}
        \end{align*}
        % \begin{itemize}
        %     \item For OPE Task: 
        %     \begin{align*}
        %         V_K \leftarrow \|Q_K \|_{ \rho^{\pi}} := \sqrt{ \mathbb{E}_{\rho(s) \pi(a|s)} \left[ Q_K(s,a)^2 \right]}
        %     \end{align*}
        %     \item For OPL Task: 
        %     \begin{align*}
        %         \pi_K(\cdot|s) \leftarrow \argmax_{a} Q_K(a|s)
        %     \end{align*}
        % \end{itemize}

        % $V_K \leftarrow \|Q_K \|_{ \rho^{\pi}} = \sqrt{ \mathbb{E}_{\rho(s) \pi(a|s)} \left[ Q_K(s,a)^2 \right]}$ for OPE Task, and $\pi_K(\cdot|s) \leftarrow \argmax_{a} Q_K(a|s)$ for OPL Task. 
        % % \State If \textbf{OPE}:
        % %   $V_K \leftarrow \|Q_K \|_{ \rho^{\pi}} = \sqrt{ \mathbb{E}_{\rho(s) \pi(a|s)} \left[ Q_K(s,a)^2 \right]}$ \; 
           
        % \State If \textbf{OPL}: $\pi_K(\cdot|s) \leftarrow \argmax_{a} Q_K(a|s)$ 
        \end{algorithmic}
        \end{algorithm}

In this work, we study a variant of fitted Q-iteration (FQI) algorithm for offline RL, presented in Algorithm \ref{alg:LSVI}. This algorithm is appealingly simple as it iteratively constructs $Q$-estimate from the offline data and the previous $Q$-estimate, as in Algorithm \ref{alg:LSVI}. 
This FQI-style algorithm has been largely studied for offline RL, such as \cite{DBLP:journals/jmlr/MunosS08,DBLP:conf/icml/ChenJ19,duan2021risk} to name a few; yet there has been no work studying this algorithm in offline RL with neural network function approximation except \cite{DBLP:journals/corr/abs-1901-00137}. However, \citet{DBLP:journals/corr/abs-1901-00137} use data splitting and rely on a more limited dynamic condition than ours. Thus, the notable difference in Algorithm \ref{alg:LSVI} is the use of neural network to approximate $Q$-functions and we estimate each $Q_k$ using the entire offline data set, instead of splitting the data into disjoint sets as in \cite{DBLP:journals/corr/abs-1901-00137}. In particular, \citet{DBLP:journals/corr/abs-1901-00137} split the offline data into $K$ disjoint sets, resulting in the sample complexity linearly scaled with $K$, which is highly inefficient 
% \st{as $K$ is often large in practice} 
{in long (effective) horizon problems where the effective horizon length $H = 1/(1-\gamma)$ is large.} 
% \thanh{I don't understand the reason we don't return $Q^{\pi}$ with the Bellman equation in Line 7 of the Algorithm 1.}

As we do not split the data into disjoint sets, a correlated structure is induced. Specifically, at each iteration $k$ in Algorithm \ref{alg:LSVI}, {$Q_{k-1}$ also depends on $(s_i, a_i)$ which makes $\sE \left[r_i + \gamma \max_{a} Q_{k-1}(s'_i, a)  \right] \neq [T^*Q_{k-1}](s_i, a_i)$ in OPL Task (and $\sE \left[r_i + \gamma \mathbb{E}_{a \sim \pi(\cdot|s'_i)} \left[ Q_{k-1}(s'_i, a) \right] \right] \neq [T^{\pi}Q_{k-1}](s_i, a_i)$ in OPE Task, respectively).}  
% $\mathbb{E} \left[ [T^*Q_{k-1}](x_i) - y_i | x_i \right] \neq 0$
% the response variable $y_i$ depends on the covariate $x_i := (s_i, a_i)$ as it depends on $Q_{k-1}$. the response variable $y_i$ is no longer centered at $[T^*Q_{k-1}](x_i)$ for OPL (or at $[T^{\pi}Q_{k-1}](x_i)$ for OPE, respectively), i.e., $\mathbb{E} \left[ [T^*Q_{k-1}](x_i) - y_i | x_i \right] \neq 0$. 
This correlated structure hinders a direct use of the standard concentration inequalities {such as Bernstein's inequality that require a sequence of random variables to adapt to certain filtration}. We overcome this technical difficulty using uniform convergence argument. 
% \thanh{The claim on the hardness of the correlated data claimed in Section 4 is hard to follow. I guess the authors would like to argue that  is estimated with  in the last iteration, but we can decouple this dependency via a uniform convergence on $Q$}

In our analysis, we assume access to the minimizer of the optimization in Algorithm \ref{alg:LSVI}. In practice, we can use (stochastic) gradient descent to effectively solve this optimization {with $L_0$ regularization \citep{DBLP:journals/corr/abs-1712-01312}}. {If the $L_0$ constraint is relaxed in practice}, (stochastic) gradient descent is guaranteed to converge to a global minimum under certain structural assumptions \citep{DBLP:conf/icml/DuLL0Z19,DBLP:conf/iclr/DuZPS19,allen2019convergence,nguyen2021proof}.
% \thanh{I’m not aware of any provable algorithms within polynomial time for optimizing sparse neural networks (with  constraints), hence the claim in Section 4 is not proper. I hope the authors can provide references on that.}

%% file: body/theory.tex
\section{Main Result}
\label{section:main_results}
% \subsection{Assumptions}
To obtain a non-trivial guarantee, certain assumptions on the distribution shift and the MDP regularity are necessary. 
% Here, we introduce such assumptions. The first assumption is a common restriction that quantifies the distribution shift in offline RL.
% We remark that a restriction in MDPs is necessary to guarantee polynomial sample complexity \citep{DBLP:conf/icml/ChenJ19}. 
We introduce the assumptions about the data generation in Assumption \ref{assumption:concentration_coefficient} and the regularity of the underlying MDP \ref{assumption:completeness}.

\begin{assumption}[\textit{Uniform concentrability coefficient \citep{DBLP:journals/jmlr/MunosS08}}]
\label{assumption:concentration_coefficient} $\exists \kappa_{\mu} < \infty$ such that $ \displaystyle \left\| \frac{d\nu}{d\mu} \right\|_{\infty} \leq \kappa_{\mu}$ for any \textit{admissible} distribution $\nu$. \footnote{$\nu$ is said to be admissible if there exist $t \geq 0$ and policy $\bar{\pi}$ such that $\nu(s,a) = \mathbb{P}(s_t = s, a_t=a|s_1 \sim \rho, \bar{\pi}), \forall s,a$.} 
\end{assumption}
\noindent The finite $\kappa_{\mu}$ in Assumption \ref{assumption:concentration_coefficient} asserts that the sampling distribution $\mu$ is not too far away from any admissible distribution, which holds for a reasonably large class of MDPs, e.g., for any finite MDP, any MDP with bounded transition kernel density, and equivalently any MDP whose top-Lyapunov exponent is negative. We present a simple (though stronger than necessary) example for which Assumption \ref{assumption:concentration_coefficient} holds. 
\begin{eg}
% If the transition density $P(s'|s,a)$ is sufficiently stochastic and the behaviour policy $\eta$ has a sufficient uniform coverage over the action space, i.e.,
If there exist absolute constants $c_1, c_2> 0$ such that for any $s,s' \in \mathcal{S}$, there exists an action $a \in \mathcal{A}$ such that $P(s'|s,a) \geq 1/c_1$ and $\eta(a|s) \geq 1 / c_2, \forall s,a$, then we can choose $\kappa_{\mu} = c_1 c_2$. 
\end{eg}

\citet{DBLP:conf/icml/ChenJ19} further provided natural problems with  rich observations generated from hidden states that has a low concentrability coefficient. These suggest that low concentrability coefficients can be found in fairly many interesting problems in practice.

% [In discussion, say that it is not clear in general cases if the dependence on the concentration coefficient is necessary].
% \begin{assumption}[\textbf{Realizability}]
% We assume that $Q^{\pi} \in \mathcal{F}_{NN}$. 
% \end{assumption}
%
% \thanh{Is this really needed?}
% Specifically, we are interested in upper bounding the quantity $\|Q^{\pi}_K - Q^{\pi} \|_{\rho}$ where $Q_{K}^{\pi}(s,a)$ is returned by Algorithm \ref{alg:FQE} and $\rho \in \mathcal{P}(\mathcal{S} \times \mathcal{A})$ is the initial state-action distribution. We specify the function class $\mathcal{F}$ in Algorithm \ref{alg:FQE} to be the class of ReLU network functions $\mathcal{F}(d, V_{\max})$. In what follows, we present all the assumptions necessary to establish our result.  
% Next, we introduce a completeness assumption.
% \st{Assumption [\textit{Besov dynamic closure}]} $\hcancel{\forall f \in \mathcal{F}_{NN}, \forall \pi, T^{\pi}f \in 
% \bar{B}^{\alpha}_{p,q}}$ 
% \st{for some} $ \hcancel{p,q \in [1, \infty], \alpha > \frac{d}{\min\{p,2\}}}$.
% \label{assumption:completeness_old}
% \end{assumption}
We now state the assumption about the regularity of the underlying MDP.
\begin{assumption}[\textit{Besov dynamic closure}]
Consider some fixed $p,q \in [1, \infty]$ and $\alpha > \frac{d}{\min\{p,2\}}$. 
\begin{itemize}
    \item \underline{For OPE Task}: For a target policy $\pi$, and for some $(L,m,S,B) \in \sN \times \sN \times \sN \times \sR_{+}$ (which will be specified later) we assume that: $\forall f \in \fnn(L,m,S,B) \implies T^{\pi} f \in \bar{B}^{\alpha}_{p,q}$. 
    \item \underline{For OPL Task}: For some $(L,m,S,B) \in \sN \times \sN \times \sN \times \sR_{+}$ (which will be specified later) we assume that: $\forall f \in \fnn(L,m,S,B) \implies T^* f \in \bar{B}^{\alpha}_{p,q}$.
\end{itemize}
\label{assumption:completeness}
\end{assumption}

% \thanh{check if I really need $\forall \pi$ or just target $\pi$}
% \thanh{Assumption 5.2 is just the Bellman completeness assumption that is widely used in reinforcement learning with general function approximation. I would claim that this assumption is in fact hiding lots of the issues, and make the original problem back to a statistical estimation problem with Besov space.}

Assumption \ref{assumption:completeness} signifies that {for OPL task (for OPE task with target policy $\pi$, respectively) the Bellman operator $T^*$ ($T^{\pi}$, respectively) applied on any ReLU network function in $\fnn(L,m,S,B)$ results in a new function that sits in $\bar{B}^{\alpha}_{p,q}(\mathcal{X})$.} 
% Note that, as $T^{\pi_f} f = T^* f$ where $\pi_f$ is the greedy policy w.r.t. $f$, Assumption \ref{assumption:completeness} also implies that $T^*f \in \bar{B}^{\alpha}_{p,q}(\mathcal{X})$ if $f \in \mathcal{F}_{NN}(\mathcal{X})$. 
The smoothness constraint $\alpha > \frac{d}{\min\{p,2\}}$ is necessary to guarantee the compactness and the finite (local) Rademacher complexity of the Besov space, and $\alpha - d/p$ is called the \textit{differential dimension} of the Besov space. Note that when $p < 2$ (thus the condition above becomes $\alpha > d/p$), a function in the corresponding Besov space contains both spiky parts and smooth parts, i.e., the Besov space has {inhomogeneous} smoothness \citep{suzuki2018adaptivity}. 
% and  $\alpha - d/p$ is called the \textit{differential dimension} of the Besov space. 

% This kind of completeness assumption is relatively standard and common in the offline RL literature \citep{DBLP:conf/icml/ChenJ19}; yet 

Our Besov dynamic closure is sufficiently general that subsumes almost all the previous completeness assumptions in the literature. For example,  a simple (yet considerably stronger than necessary) sufficient condition for Assumption \ref{assumption:completeness} is that the expected reward function $r(s,a)$ and the transition density $P(s'|s,a)$ for each fixed $s'$ are functions in $\bar{B}^{\alpha}_{p,q}$, regardless of any input function $f$ and any target policy $\pi$. \footnote{This sufficient condition imposes the smoothness constraint solely on the underlying MDP regardless of the input function $f$. Thus, the ``max'' over the input function $f(s,a)$ does not affect the smoothness of the resulting function after $f$ is passed through the Bellman operator. This holds regardless of whether $f$ is in the Besov space.} Such a condition on the transition dynamic is common in the RL literature; for example, linear MDPs \citep{jin2020provably} posit a linear structure on the expected reward and the transition density as $r(s,a) = \langle \phi(s,a), \theta \rangle$ and $P(s'|s,a) = \langle \phi(s,a), \lambda(s') \rangle$ for some feature map $\phi: \mathcal{X} \rightarrow \mathbb{R}^{d_0}$ and signed measures $\lambda(s') = (\lambda(s')_1, \ldots, \lambda(s')_{d_0})$.  To make it even more concrete, we present the {following} examples for Assumption \ref{assumption:completeness}. 

\begin{eg}[\textit{Reproducing kernel Hilbert space (RKHS)}]
Define $k_{h, l}$ the Mat\'ern kernel with smoothness parameter $h > 0$ and length scale $l > 0$. If both $r(\cdot)$ and $g_{s'}( \cdot ) := P(s' | \cdot)$ at any $s' \in \mathcal{S}$ are functions in the RKHS of Mat\'ern kernel $k_{h, l}$ where $h = \alpha - d/2 > 0$ and $l > 0$, then Assumption \ref{assumption:completeness} holds for $p=q=2$. \footnote{This is due to the norm-equivalence between the above RKHS and the Sobolev space $W^{\alpha}_2(\mathcal{X})$ \citep{kanagawa2018gaussian} and the degeneration from Besov spaces to Sobolev spaces as $B^{\alpha}_{2,2}(\mathcal{X}) = W^{\alpha}_2(\mathcal{X})$.} Moreover, this particular case is equivalent to the dynamic condition considered in \cite{DBLP:journals/corr/abs-1901-00137}. 
\end{eg}

% \begin{eg}[\textit{\blue{Inhomogeneous smoothness functions in $L_p(\gX)$}}]
% When $p < 2$ (thus the constraint condition becomes $\alpha > d/p$), the expected reward $r(\cdot)$ and the transition densities $g_{s'}( \cdot ) := P(s' | \cdot)$ contain both spiky parts and smooth parts, i.e., inhomogeneous smoothness. \\ %\citep{suzuki2018adaptivity}. \\
% % , and  $\alpha - d/p$ is called the \textit{differential dimension} of the Besov space.
% \end{eg}
% \thanh{Example 5.3 is not an example.}

\begin{eg}[Reduction to linear MDPs]
Linear MDPs \citep{jin2020provably} correspond to Assumption \ref{assumption:completeness} with $\alpha=1$ and $p=q$ on a $p$-norm bounded domain. \footnote{However, linear MDPs do not require the smoothness constraint $\alpha > \frac{d}{\min\{p, 2\}}$ to ensure a finite Rademacher complexity of linear models. Of course, our analysis addresses significantly more complex and general settings than linear MDPs which we believe is more important than recovering this particular condition of linear MDPs.}
\end{eg}

{Note that Assumption \ref{assumption:completeness} even allows the expected rewards $r(\cdot)$ and the transition densities $g_{s'}( \cdot ) := P(s' | \cdot)$ to contain both spiky parts and smooth parts, i.e., inhomogeneous smoothness, as long as $p < 2$ (thus the constraint condition becomes $\alpha > d/p$).} 
\noindent We are now ready to present our main result. 
% under Assumption \ref{assumption:concentration_coefficient} and Assumption \ref{assumption:completeness}. 
% Before stating the main result, we introduce the necessary notations of asymptotic relations. For any two real-valued functions $f$ and $g$, we write $f(\epsilon, n) \lesssim g(\epsilon,n)$ if there is an absolute constant $c$ such that $f(\epsilon, n) \leq c \cdot g(\epsilon, n), \forall \epsilon > 0, n \in \mathbb{N}$. We write $f(\epsilon, n) \asymp g(\epsilon, n)$ if $f(\epsilon, n) \lesssim g(\epsilon, n)$ and $g(\epsilon, n) \lesssim f(\epsilon, n)$. We write $f(\epsilon, n) \simeq g(\epsilon, n)$ if there exists an absolute constant $c$ such that $f(\epsilon, n) = c \cdot g(\epsilon, n), \forall \epsilon, n$. 

% Note that $f(\epsilon, n) \simeq g(\epsilon, n)$ implies $f(\epsilon, n) \asymp g(\epsilon, n)$.
\begin{thm}
Under Assumption \ref{assumption:concentration_coefficient} and Assumption \ref{assumption:completeness} {for some $(L,m,S,B)$ satisfying (\ref{eq: network architecture specs})}, for any $\epsilon > 0, \delta \in (0,1], K > 0$, if $n$ satisfies that $n \gtrsim \left(\frac{1}{\epsilon^2} \right)^{1 + \frac{d}{\alpha}} \log^6 n + \frac{1}{\epsilon^2}(\log(1/\delta) + \log \log n)$, then with probability at least $1 - \delta$, the sup-optimality of Algorithm \ref{alg:LSVI} is 
\begin{align*}
    \begin{cases}
    \displaystyle \subopt(V_K; \pi) \leq \frac{ \sqrt{ \kappa_{\mu}}}{1-\gamma} \epsilon + \frac{ \gamma^{K/2}}{(1-\gamma)^{1/2}} &\text{ for OPE}, \\ 
    \displaystyle \subopt(\pi_K) \leq \frac{2 \gamma \sqrt{ \kappa_{\mu}}}{(1-\gamma)^2} \epsilon + \frac{2 \gamma^{1 + K/2}}{(1-\gamma)^{3/2}} &\text{ for OPL.}
    \end{cases}
\end{align*}
In addition, the optimal deep ReLU network $\Phi(L,m,S,B)$ that obtains such sample complexity (for both OPE and OPL) satisfies 
\begin{align}
     L \asymp \log N, m \asymp N \log N, S \asymp N, \text{ and } B \asymp N^{1/d + (2 \iota)/ (\alpha - \iota) }, 
     \label{eq: network architecture specs}
\end{align}
where 
% \begin{align*}
    $N \asymp n^{\frac{ 1/2 + \left(2 + d^2 / (\alpha(\alpha + d)) \right)^{-1} }{1 + 2 \alpha /d } } \text{ and } \iota := d(p^{-1} - (1 + \floor{\alpha})^{-1})_{+}$. 
% \end{align*}
% and $$. 
\label{thm:sample_complexity}
% is the number of parameters to approximate a function in the Besov space, and  
\end{thm}
% \thanh{The proof techniques from \citep{DBLP:journals/jmlr/MunosS08,DBLP:conf/icml/0002VY19} cannot apply here}

As the complete form of Theorem \ref{thm:sample_complexity} is quite involved, we interpret and disentangle this result to understand FQI algorithms with neural network function approximation for offline RL tasks. The sub-optimality in both OPE and OPL consists of the statistical error (the first term) and the algorithmic error (the second term). While the algorithmic error enjoys the fast linear convergence to $0$ as $K$ gets large, the statistical error reflects the fundamental difficulty of our problems. To make it more interpretable, we present a simplified version of Theorem \ref{thm:sample_complexity} where we state the sample complexity required to obtain a sub-optimality within $\epsilon$. 

\begin{prop}[Simplified version of Theorem \ref{thm:sample_complexity}]
For any {$K \gtrsim H \log(1/\eps)$}, the sample complexity of Algorithm \ref{alg:LSVI} for OPE Task and OPL Task is {$n = \tilde{\mathcal{O}}( H^{2 + 2 \frac{d}{\alpha}} \kappa_{\mu}^{1 + \frac{d}{\alpha}} \epsilon^{-2 - 2\frac{d}{\alpha}} )$} and $n = \tilde{\mathcal{O}}( H^{4 + 4 \frac{d}{\alpha}} \kappa_{\mu}^{1 + \frac{d}{\alpha}} \epsilon^{-2 - 2\frac{d}{\alpha}} )$, respectively.
% \footnote{Here $\tilde{\mathcal{O}}$ ignores any log factor and the factor pertaining $1/(1-\gamma)$.} 
Moreover, the optimal deep ReLU network $\Phi(L,m,S,B)$ {for both OPE and OPL Tasks} that obtains such sample complexity is 
% \begin{align*}
    $L = \mathcal{O}(\log n), m = \mathcal{O}(n^{2/5}\log n), S = \mathcal{O}(n^{2/5}), \text{ and } \log B = \mathcal{O}(\frac{n^{2/5}}{d} ). $
% \end{align*}
\label{prop: simplied version of the main theorem}
\end{prop}

To discuss our result, we compare it with other existing works in Table \ref{tab:compare_literature}. 
As the literature of offline RL is vast, we only compare with representative works of FQI estimators for offline RL with function approximation under a uniform data coverage assumption, as they are directly relevant to our work that uses FQI estimators with neural network function approximation under uniform data coverage. Here, our sample complexity does not scale with the number of states as in tabular MDPs \citep{DBLP:conf/aistats/YinW20,DBLP:conf/aistats/YinBW21,yin2021characterizing} or the inherent Bellman error as in the general function approximation \citep{DBLP:journals/jmlr/MunosS08,DBLP:conf/icml/0002VY19,duan2021risk}. Instead, it explicitly scales with the (possible fractional) smoothness $\alpha$ of the underlying MDP, the dimension $d$ of the input space, {the distributional shift measure $\kappa_{\mu}$ and the effective episode length $H = (1- \gamma)^{-1}$}. Importantly, this guarantee is established under the Besov dynamic closure that subsumes the dynamic conditions of the prior results. 
% \st{The sample complexity in} \citep{DBLP:journals/corr/abs-1901-00137} \st{ linearly scales with} $K$ \st{which becomes vacuous as} $\hcancel{K \rightarrow \infty}$. \st{On the other hand, our sample complexity avoids such vacuous bound by considering the correlated structure in the value regression (i.e., data reuse instead of data splitting). This improvement comes at an extra cost of} $\hcancel{\epsilon^{-d/\alpha}}$ \st{that we pay for the correlated structure in value regression. While this extra cost may appear large at first, it in fact scales with} $1 / \epsilon^{\max\{1/2, 1/p\}}$ \st{which is significantly smaller than} $K$. 
{Compared to \cite{DBLP:journals/corr/abs-1901-00137}, our sample complexity has a strong advantage in long (effective) horizon problems where $H > \frac{d}{\alpha - 2d} \log(1/\eps)$ \footnote{This condition is often easily satisfied as in practice we commonly set $\gamma = 0.99$ and $\eps = 0.001$, thus we have $H=100$ and $\log(1/\eps) = 3$.} and improves it by a factor of $H^{1 - 2d/\alpha} \eps^{-d/ \alpha} \log(H^2/\eps^2)$. It also suggests that the data splitting in \cite{DBLP:journals/corr/abs-1901-00137} should be preferred for short (effective) horizon problems.} Though our bound has a tighter dependence on $H$ in the long horizon setting, the dependence on $\eps$ in our bound is compromised and does not match the minimax rate in the regression setting. We leave as future direction to construct the lower bound for the data-reuse setting of offline RL.

\paragraph{On the role of deep ReLU networks in offline RL.}
We make several remarks about the role of deep networks in offline RL. The role of deep ReLU networks in offline RL is to guarantee a maximal adaptivity to the (spatial) regularity of the functions in Besov space and obtain an optimal approximation error rate that otherwise were not possible with other function approximation such as kernel methods \citep{suzuki2018adaptivity}. Moreover, by the equivalence in the functions that a neural architecture can compute \citep{yarotsky2017error}, Theorem \ref{thm:sample_complexity} also readily holds for any other continuous piece-wise linear activation functions with finitely many line segments $M$ where the optimal network architecture only increases the number of units and weights by constant factors depending only on $M$. Moreover, we observe that the optimal ReLU network is relatively ``thinner'' than overparameterized neural networks that have been recently studied in the literature \citep{arora2019exact,allen2019convergence,hanin2019finite,cao2019generalization,belkin2021fit} where the width $m$ is a high-order polynomial of $n$. As overparameterization is a key feature for such overparameterized neural networks to obtain a good generalization, it is natural to ask why a thinner neural network in Theorem \ref{thm:sample_complexity} also guarantees a strong generalization for offline RL? Intuitively, the optimal ReLU network in Theorem \ref{thm:sample_complexity} is regularized by a strong sparsity which resonates with our practical wisdom that a sparsity-based regularization prevents over-fitting and achieve a better generalization. Indeed, as the total number of parameters in the considered neural network is $p = md + m + m^2(L - 2) = \mathcal{O}(N^2 \log^3 N)$ while the number of non-zeros parameters $S$ only scales with $N$, the optimal ReLU network in Theorem \ref{thm:sample_complexity} is relatively sparse.

%% file: body/technical_review.tex
\section{Technical Review}
In this section, we highlight the key technical challenges
% \st{and novelty} 
in our analysis. In summary, two key technical challenges in our analysis are rooted in the consideration of the correlated structure in value regression in Algorithm \ref{alg:LSVI}, and the use of deep neural network as function approximation (and their combination). To address these challenges, {we devise}
% On the other hand, the technical novelty in our analysis is
the so-called double uniform convergence argument and {leverage} a localization argument via sub-root functions for local Rademacher complexities. In what follows, we briefly discuss these technical challenges and 
% \st{novelty of our analysis} 
{our analysis approach}. 

% Let $y_i^k$ be the $y_i$ computed at iteration $k$ in a FQI-style algorithm (e.g., Algorithm \ref{alg:LSVI}). 
The analysis and technical proofs of \cite{DBLP:journals/corr/abs-1901-00137,DBLP:conf/icml/0002VY19} heavily rely on the equation $\sE \left[r_i + \gamma \mathbb{E}_{a' \sim \pi(\cdot|s'_i)} \left[ Q_{k-1}(s'_i, a) \right] \right] = [T^*Q_{k-1}](s_i, a_i)$
% $\mathbb{E}[y_i^k | x_i] = \mathbb{E}[[T^* Q_{k-1}](x_i)]$
to leverage the standard nonparametric regression techniques (in a supervised learning setting). However, the correlated structure in Algorithm \ref{alg:LSVI} implies $\sE \left[r_i + \gamma \mathbb{E}_{a' \sim \pi(\cdot|s'_i)} \left[ Q_{k-1}(s'_i, a) \right] \right] \neq [T^*Q_{k-1}](s_i, a_i)$ as $Q_{k-1}$ also depends on $(s_i, a_i)$. 
% that $\mathbb{E}[y_i^k | x_i] \neq \mathbb{E}[[T^* Q_{k-1}](x_i)]$ as $Q_{k-1}$ also depends on $x_i$. 
Thus, the techniques in these prior works could not be used here and we require a new analysis. It is worth noting that \citet{DBLP:conf/icml/0002VY19} also re-use the data as in Algorithm \ref{alg:LSVI} (instead of data splitting as in \cite{DBLP:journals/corr/abs-1901-00137}) but mistakenly assume that $\sE \left[r_i + \gamma \mathbb{E}_{a' \sim \pi(\cdot|s'_i)} \left[ Q_{k-1}(s'_i, a) \right] \right] = [T^*Q_{k-1}](s_i, a_i)$. 
% Thus, their analysis is also improper. 
To deal with the correlated structure, we 
% \st{propose a new analysis technique, namely}
{devise} a \textit{double} uniform convergence argument. The double uniform convergence argument is appealingly intuitive: while in a standard regression problem, the (single) uniform convergence argument seeks the generalization guarantee uniformly over an entire hypothesis space of a data-dependent empirical risk minimizer, in the value regression problem of Algorithm \ref{alg:LSVI}, we additionally guarantee generalization uniformly over the hypothesis space of the data-dependent regression target $T^*Q_{k-1}$. 
% since $T^*Q_{k-1}$ is data-dependent. 
To make it concrete, we highlight a key equality in our proof where the double uniform convergence argument is used: 
\begin{align*}
    \max_{k} \| Q_{k+1} - T^*Q_k \|_{\mu}^2 = \underbrace{\sup_{Q \in \fnn} (\mathbb{E} - \mathbb{E}_n)(l_{\hat{f}^Q} - l_{f_*^Q})}_{I_1, \text{empirical process term}} + \underbrace{\sup_{Q \in \fnn} \mathbb{E}_n(l_{f_{\perp}^Q} - l_{f_*^Q})}_{I_2, \text{bias term}}, 
\end{align*}
where $f_*^Q(x) = \mathbb{E}[r + \gamma  \max_{a'} Q(s', a')|x]$, and  $f_{\perp}^Q := \arginf_{f \in \fnn} \|f - f_*^Q \|_{2, \mu}$, and {$l_{f_{\perp}^Q} := (f_{\perp}^Q(x_1) - r_1 - \gamma  \max_{a'} Q(s'_1, a'))^2$ and $l_{f_{*}^Q} := (f_{*}^Q(x_1) - r_1 - \gamma  \max_{a'} Q(s'_1, a'))^2$ are random variables with respect to the randomness of $(x_1, s'_1, r_1).$} {We have learned that a similar general idea of the double uniform convergence argument has been leveraged in \cite{chen2019information} for general function classes. We remark they use finite function classes, and in our case, the double uniform convergence argument is particularly helpful in dealing with local Rademacher complexities under a data-dependent structure as local Rademacher complexities already involve the supremum operator which can be naturally incorporated with the double uniform convergence argument.}

% \thanh{introduce $l$}

The double uniform convergence argument also requires a different technique to control an empirical process term $I_1$ as it now requires a uniform convergence over the regression target. We leverage local Rademacher complexities to derive a bound on $I_1$:
\begin{align*}
    &\sup \{(\mathbb{E} - \mathbb{E}_n)(l_{\hat{f}^Q} - l_{f_*^Q}): Q \in \fnn, \|\hat{f}^Q - f_*^Q\|^2_{\mu} \leq r \} \\
    % &\leq \sup \{(\mathbb{E} - \mathbb{E}_n)(l_f - l_g): f \in \fnn, g \in T^{*} \mathcal{F}, \|f - g\|^2_{\mu} \leq r \} \\
     &\leq 6 \mathbb{E} R_n \left\{f - g: f \in \fnn, g \in T^{*} \fnn, \|f - g \|_{\mu}^2 \leq r \right \} + 2 \sqrt{ \frac{2 r \log(1/\delta) }{n} } + \frac{28 \log(1/\delta)}{3n}. 
\end{align*}
where $R_n$ is the local Rademacher complexity \citep{bartlett2005}. An explicit bound is then derived via a localization argument and the fixed point of a sub-root function. 

The use of neural networks pose a new challenge mainly in bounding the bias term $I_2$. We derive this bound using the adaptivity of deep ReLU network to the regularity in Besov spaces, leveraging our Besov dynamic condition in Assumption \ref{assumption:completeness}. Bounding the bias term also requires the use of a concentration inequality. While \citet{DBLP:conf/icml/0002VY19} use Bernstein's inequality, our bias term $I_2$ requires a uniform convergence version of Bernstein's inequality as $I_2$ requires a guarantee uniformly over $\fnn$.  We omit a detailed proof for Theorem \ref{thm:sample_complexity} to Section \ref{Appendix:A}. 

%% file: body/conclusion.tex
\section{Conclusion and Discussion}
\label{section:discussion}
We presented the sample complexity of FQI estimators for offline RL with deep ReLU network function approximation under a uniform data coverage assumption. We proved that the FQI-type algorithm achieved the sample complexity of $n = \tilde{\mathcal{O}}( H^{4 + 4 \frac{d}{\alpha}} \kappa_{\mu}^{1 + \frac{d}{\alpha}} \epsilon^{-2 - 2\frac{d}{\alpha}})$ under a correlated structure and a general dynamic condition namely the Besov dynamic closure. In addition, we corrected the mistake in ignoring the correlated structure when reusing data with FQI estimators in \cite{DBLP:conf/icml/0002VY19}, avoided the possibly inefficient data splitting technique in \cite{DBLP:journals/corr/abs-1901-00137} for {long (effective) horizon problems}, and proposed a general dynamic condition that subsumes all the previous Bellmen completeness assumptions. In the following, we discuss future directions.

% We also provided various insights into the benefits and the effects of deep neural networks in offline RL. 
% with the data-dependent structure obtains an improved sample complexity of $\tilde{\mathcal{O}}\left(  \kappa^{1 + d/\alpha} \cdot \epsilon^{-2 - 2d/\alpha} \right)$ under a standard condition of distributional shift and a new dynamic condition namely Besov dynamic closure which encompasses the dynamic conditions considered in the prior results. Established under the data-dependent structure and the general Besov dynamic closure, our sample complexity is the most general result for offline RL with deep ReLU network function approximation. 

% The generality of our result under the general dynamic condition and the data-dependent structure 

% These conditions generalize the previous H\"older smoothness assumptions studied in recent work. Importantly, our sample complexity is established under the complicated data-dependent structure, a technical bottleneck of offline RL that is either avoided or improperly handled in prior work. 

\paragraph{Relaxing the assumption about uniform data coverage.} 
For a future work, we can include the pessimistic approach in \cite{jin2020pessimism,rashidinejad2021bridging,uehara2021pessimistic,nguyen2021offline} to the current work with a more involved analysis of uncertainty quantifiers under non-linear function approximation to relax the strictness of the uniform data coverage assumption. 

\paragraph{Relaxing the assumption about optimization oracle.} The present work assumes access to the optimization oracle when fitting a neural network to the offline data. It is desirable to understand how optimization and generalization of a trained neural network can contribute to offline RL with neural function approximation. A promising approach to obtain a tight trajectory-dependent sub-optimality bound of offline RL with neural function approximation is to characterize the SGD-based optimization via a stochastic differential
equation by allowing the stochastic noises to follow the fractional Brownian motion \cite{tan2022trajectory,tong2022learning}.  

% A recent tool that might be helpful to analyze this setting is trajectory-dependent generalization of deep neural networks using fractional Brownian motions. 

\section*{Acknowledgements}
We thank our anonymous reviewers and our action editor Yu-Xiang Wang (UC Santa Barbara) at TMLR for the constructive comments and feedback. Thanh Nguyen-Tang thank Le Minh Khue Nguyen (University of Rochester) for the support of this project during the COVID times. 

%% file: body/appendixA.tex
\renewcommand{\thesection}{A}
\section{Proof of Theorem \ref{thm:sample_complexity}}
\label{Appendix:A}

We now provide a complete proof of Theorem \ref{thm:sample_complexity}. The proof has four main components: a sub-optimality decomposition for error propagation across iterations, a Bellman error decomposition using a uniform convergence argument, a deviation analysis for least squares with deep ReLU networks using local Rademacher complexities and a localization argument, and a upper bound minimization step to obtain an optimal deep ReLU architecture. 

\subsubsection*{Step 1: A sub-optimality decomposition}
The first step of the proof is a sub-optimality decomposition, stated in Lemma \ref{sub_opt_decompose}, that applies generally to any least-squares Q-iteration methods. 
\begin{lem}[\textit{A sub-optimality decomposition}]
Under Assumption \ref{assumption:concentration_coefficient},
the sub-optimality of $V_K$ returned by Algorithm \ref{alg:LSVI} is bounded as 
\begin{align*}
    \subopt(V_K)  \leq
    \begin{cases}
    \frac{\sqrt{\kappa_{\mu}}}{1-\gamma} \displaystyle \max_{0 \leq k \leq K-1} \| Q_{k+1} - T^{\pi} Q_k \|_{\mu} + \frac{ \gamma^{K/2}}{(1-\gamma)^{1/2}} &\text{ for OPE}, \\ 
    \frac{4 \gamma \sqrt{\kappa_{\mu}}}{(1-\gamma)^2} \displaystyle \max_{0 \leq k \leq K-1} \| Q_{k+1} - T^* Q_k \|_{\mu} + \frac{4 \gamma^{1 + K/2}}{(1-\gamma)^{3/2}} &\text{ for OPL}.
    \end{cases}
\end{align*}
% where we denote $\|f\|_{\mu} := \sqrt{ \int \mu(ds da) f(s,a)^2 }, \forall f: \mathcal{S} \times \mathcal{A} \rightarrow \mathbb{R}$. 
\label{sub_opt_decompose}
\end{lem}
The lemma states that the sub-optimality decomposes into a statistical error (the first term) and an algorithmic error (the second term). While the algorithmic error enjoys the fast linear convergence rate, the statistical error arises from the distributional shift in the offline data and the estimation error of the target $Q$-value functions due to finite data. Crucially, the contraction of the (optimality) Bellman operators $T^{\pi}$ and $T^*$ allows the sup-optimality error at the final iteration $K$ to propagate across all iterations $k \in [0,K-1]$. Note that this result is agnostic to any function approximation form and does not require Assumption \ref{assumption:completeness}. The result uses a relatively standard argument that appears in a number of works on offline RL \citep{DBLP:journals/jmlr/MunosS08,DBLP:conf/icml/0002VY19}. 
\begin{proof}[Proof of Lemma \ref{sub_opt_decompose}]
We will prove the sup-optimality decomposition for both settings: OPE and OPL.  

\paragraph{(i) For OPE.} 
We denote the right-linear operator by $P^{\pi} \cdot : \{\mathcal{X} \rightarrow \mathbb{R} \} \rightarrow \{\mathcal{X} \rightarrow \mathbb{R} \}$ where
\begin{align*}
    (P^{\pi}f)(s,a) := \int_{\mathcal{X}} f(s',a')\pi(da'|s') P(ds'|s,a), 
\end{align*}
for any $f \in \{\mathcal{X} \rightarrow \mathbb{R} \}$. Denote Denote $\rho^{\pi}(ds da) = \rho(ds) \pi(da|s)$. Let $\epsilon_k := Q_{k+1} - T^{\pi} Q_k, \forall k \in [0,K-1]$ and $\epsilon_K = Q_0 - Q^{\pi}$. Since $Q^{\pi}$ is the (unique) fixed point of $T^{\pi}$, we have 
\begin{align*}
    Q_k - Q^{\pi} &= T^{\pi} Q_{k-1} - T^{\pi} Q^{\pi} + \epsilon_{k-1} = \gamma P^{\pi}(Q_{k-1} - Q^{\pi}) + \epsilon_{k-1}.
\end{align*}
By recursion, we have 
\begin{align*}
     Q_K - Q^{\pi} &= \sum_{k=0}^K (\gamma P^{\pi})^k \epsilon_k = \frac{1 - \gamma^{K+1}}{1 - \gamma} \sum_{k=0}^K \alpha_k A_k \epsilon_k  
    %  \gamma^K (P^{\pi})^K (Q_0 - Q^{\pi}) + \sum_{k=0}^{K-1} \gamma^{k} (P^{\pi})^k \epsilon_{K-1-k} \\ 
    %  &= \frac{1-\gamma^{K+1}}{1-\gamma} \left( \frac{(1-\gamma) \gamma^K}{1 - \gamma^{K+1}} (P^{\pi})^K (Q_0 - Q^{\pi}) + \sum_{k=0}^{K-1} \frac{(1-\gamma) \gamma^k}{1 - \gamma^{K+1}} (P^{\pi})^k \epsilon_{K-1-k} \right) \\ 
    %  &=  \frac{1-\gamma^{K+1}}{1-\gamma} \sum_{k=0}^K \alpha_k A_k \xi_k,
\end{align*}
where $\alpha_k :=  \frac{(1-\gamma) \gamma^k}{1 - \gamma^{K+1}}, \forall k \in [K]$ and $A_k := (P^{\pi})^k,  \forall k \in [K]$. 
% \begin{align*}
%     , \text{ and } A_k := (P^{\pi})^k,  \forall k \in [K],\\
%     % \xi_k &:= 
%     % \begin{cases}
%     % \epsilon_{K-1-k} & \text{ for } k \in [K-1] \\ 
%     % Q_0 - Q^{\pi} & \text{ for } k = K.
%     % \end{cases}
% \end{align*}
Note that $\sum_{k=0}^K \alpha_k = 1$ and $A_k$'s are probability kernels. Denoting by $|f|$ the point-wise absolute value $|f(s,a)|$, we have that the following inequality holds point-wise: 
\begin{align*}
    |Q_K - Q^{\pi}| \leq \frac{1-\gamma^{K+1}}{1-\gamma} \sum_{k=0}^K \alpha_k A_k |\epsilon_k|.
\end{align*}
We have 
\begin{align*}
    \|Q_K - Q^{\pi}\|^2_{\rho^{\pi}} &\leq \frac{(1-\gamma^{K+1})^2}{(1-\gamma)^2} \int \rho(ds) \pi(da|s) \left( \sum_{k=0}^K \alpha_k A_k |\epsilon_k|(s,a)\right)^2 \\ 
    &\overset{(a)}{\leq} \frac{(1-\gamma^{K+1})^2}{(1-\gamma)^2} \int \rho(ds) \pi(da|s) \sum_{k=0}^K \alpha_k  A_k^2 \epsilon_k^2(s,a) \\
    &\overset{(b)}{\leq} \frac{(1-\gamma^{K+1})^2}{(1-\gamma)^2} \int \rho(ds) \pi(da|s) \sum_{k=0}^K \alpha_k  A_k \epsilon_k^2(s,a) \\
    &\overset{(c)}{\leq} \frac{(1-\gamma^{K+1})^2}{(1-\gamma)^2} \left( \int \rho(ds) \pi(da|s) \sum_{k=0}^{K-1}\alpha_k  A_k \epsilon_k^2(s,a) + \alpha_K \right) \\
    &\overset{(d)}{\leq} \frac{(1-\gamma^{K+1})^2}{(1-\gamma)^2} \left( \int \mu(ds,da) \sum_{k=0}^{K-1} \alpha_k  \kappa_{\mu} \epsilon_k^2(s,a) + \alpha_K \right) \\ 
    &= \frac{(1-\gamma^{K+1})^2}{(1-\gamma)^2} \left(\sum_{k=0}^{K-1} \alpha_k  \kappa_{\mu}  \| \epsilon_k \|^2_{\mu} + \alpha_K \right) \\
    & \leq \frac{\kappa_{\mu}}{(1 - \gamma)^2} \max_{0 \leq k \leq K-1} \| \epsilon_k \|_{\mu}^2 + \frac{\gamma^K}{(1 - \gamma)}.
    % &= \frac{1-\gamma^{K+1}}{(1-\gamma)^2} (1-\gamma)\sum_{k=0}^{K-1} \gamma^k \kappa_{\mu}  \| \epsilon_k \|^2_{2,\mu}  + \frac{4(1 - \gamma^{K+1}) \gamma^K}{1-\gamma} \\
    % &\leq \frac{1-\gamma^{K+1}}{(1-\gamma)^2} \kappa_{\mu} \max_{0 \leq k \leq K-1} \| \xi_k \|^2_{2,\mu} + \frac{4(1 - \gamma^{K+1})  \gamma^K}{1-\gamma} \\ &\leq \frac{1}{(1-\gamma)^2} \kappa_{\mu} \max_{0 \leq k \leq K-1} \| \xi_k \|^2_{2,\mu} + \frac{4  \gamma^K}{1-\gamma}.
\end{align*}
The inequalities $(a)$ and $(b)$ follow from Jensen's inequality, $(c)$ follows from $\|Q_0\|_{\infty}, \|Q^{\pi}\|_{\infty} \leq 1$, and $(d)$ follows from Assumption \ref{assumption:concentration_coefficient} that $\rho^{\pi} A_k = \rho^{\pi} (P^{\pi})^k \leq \kappa_{\mu} \mu$. Thus we have 
\begin{align*}
    \subopt(V_K; \pi) &= |V_K - V^{\pi}| \nonumber \\
    &= \bigg|\mathbb{E}_{\rho, \pi}[Q_K(s,a)] - \mathbb{E}_{\rho}[Q^{\pi}(s,a)] \bigg| \nonumber \\
    &\leq \mathbb{E}_{\rho, \pi} \left[ |Q_K(s,a) - Q^{\pi}(s,a)| \right] \nonumber \\
    &\leq \sqrt{   \mathbb{E}_{\rho, \pi} \left[ (Q_K(s,a) - Q^{\pi}(s,a))^2 \right] } \nonumber \\ 
    &= \|Q_K - Q^{\pi}\|_{\rho^{\pi}} \nonumber \\
    &\leq \frac{ \sqrt{\kappa_{\mu}} }{1-\gamma} \max_{0 \leq k \leq K-1} \| \epsilon_k \|_{\mu} + \frac{\gamma^{K/2}}{(1-\gamma)^{1/2}}.
\end{align*}
\paragraph{(ii) For OPL.} The sup-optimality for the OPL setting is  more complex than the OPE setting but the technical steps are relatively similar. In particular, let $\epsilon_{k-1} = T^* Q_{k-1} - Q_k, \forall k$ and $\pi^*(s) = \argmax_{a} Q^*(s,a), \forall s$, we have 
\begin{align}
    Q^* - Q_K &= T^{\pi^*} Q^* - T^{\pi^*} Q_{K-1} + \underbrace{T^{\pi^*} Q_{K-1} - T^* Q_{K-1}}_{\leq 0} + \epsilon_{K-1} \nonumber \\ 
    &\leq \gamma P^{\pi^*} (Q^* - Q_{K-1}) + \epsilon_{K-1} \nonumber \\ 
    &\leq \sum_{k=0}^{K-1} \gamma^{K-k-1} (P^{\pi^*})^{K-k-1} \epsilon_k + \gamma^{K} (P^{\pi^*})^K (Q^* - Q_0) \text{ (by recursion)}.
    \label{eq:upper_bound_Qstar_QK}
\end{align}
Now, let $\pi_k$ be the greedy policy w.r.t. $Q_k$, we have 
\begin{align}
    Q^* - Q_K &= \underbrace{T^{\pi^*} Q^*}_{\geq T^{\pi_{K-1}} Q^*} - T^{\pi_{K-1}} Q_{K-1} + \underbrace{T^{\pi_{K-1}} Q_{K-1} - T^* Q_{K-1}}_{\geq 0} + \epsilon_{K-1} \nonumber \\ 
    &\geq \gamma P^{\pi_{K-1}}(Q^* - Q_{K-1}) + \epsilon_{K-1} \nonumber \\ 
    &\geq \sum_{k=0}^{K-1} \gamma^{K - k -1} (P^{\pi_{K-1}} \ldots P^{\pi_{k+1}}) \epsilon_k + \gamma^K (P^{\pi_{K-1}} \ldots P^{\pi_0}) (Q^* - Q_0).  
    \label{eq:lower_bound_Qstar_QK}
\end{align}
Now, we turn to decompose $Q^* - Q^{\pi_K}$ as 
\begin{align*}
    Q^* - Q^{\pi_K} &= (T^{\pi^*} Q^* - T^{\pi^*} Q_K) + \underbrace{(T^{\pi^*} Q_K - T^{\pi_K} Q_K)}_{\leq 0} + (T^{\pi_K} Q_K - T^{\pi_K} Q^{\pi_K}) \\ 
    &\leq \gamma P^{\pi^*} (Q^* - Q_K) + \gamma P^{\pi_K}(Q_K - Q^* + Q^* - Q^{\pi_K}). 
\end{align*}
Thus, we have 
\begin{align*}
    (I - \gamma P^{\pi_K}) (Q^* - Q^{\pi_K}) \leq \gamma (P^{\pi^*} - P^{\pi_K}) (Q^* - Q_K) .
\end{align*}
Note that the operator $(I - \gamma P^{\pi_K})^{-1} = \sum_{i=0}^{\infty} (\gamma P^{\pi_K})^i$ is monotone, thus 
\begin{align}
    Q^* - Q^{\pi_K} \leq \gamma (I - \gamma P^{\pi_K})^{-1} P^{\pi^*} (Q^* - Q_K) - \gamma  (I - \gamma P^{\pi_K})^{-1} P^{\pi_K} (Q^* - Q_K).
    \label{eq:upper_bound_Qstar_QpiK}
\end{align}
Combining Equation (\ref{eq:upper_bound_Qstar_QpiK}) with Equations (\ref{eq:upper_bound_Qstar_QK}) and (\ref{eq:lower_bound_Qstar_QK}), we have 
\begin{align*}
    Q^* - Q^{\pi_K} &\leq (I - \gamma P^{\pi_K})^{-1} \left( \sum_{k=0}^{K-1} \gamma^{K-k} (P^{\pi^*})^{K-k} \epsilon_k + \gamma^{K+1} (P^{\pi^*})^{K+1} (Q^* - Q_0) \right) - \\ 
    &  (I - \gamma P^{\pi_K})^{-1} \left( \sum_{k=0}^{K-1} \gamma^{K - k} (P^{\pi_{K}} \ldots P^{\pi_{k+1}}) \epsilon_k + \gamma^{K+1} (P^{\pi_{K}} \ldots P^{\pi_0}) (Q^* - Q_0) \right).
\end{align*}
Using the triangle inequality, the above inequality becomes 
\begin{align*}
    Q^* - Q^{\pi_K} \leq \frac{2 \gamma (1 - \gamma^{K+1})}{(1 - \gamma)^2}  \left( \sum_{k=0}^{K-1} \alpha_k A_k |\epsilon_k| + \alpha_K A_K |Q^* - Q_0| \right),
\end{align*}
where 
\begin{align*}
    A_k &= \frac{1-\gamma}{2} (I - \gamma P^{\pi_K})^{-1} \left( (P^{\pi^*})^{K-k} + P^{\pi_{K}} \ldots P^{\pi_{k+1}} \right), \forall k < K, \\ 
    A_K &= \frac{1-\gamma}{2} (I - \gamma P^{\pi_K})^{-1} \left( (P^{\pi^*})^{K + 1} + P^{\pi_{K}} \ldots P^{\pi_0} \right), \\
    \alpha_k &= \gamma^{K-k-1} (1- \gamma) / (1-\gamma^{K+1}), \forall k < K, \\ 
    \alpha_K &= \gamma^K (1 - \gamma) / ( 1- \gamma^{K+1}). 
\end{align*}
Note that $A_k$ is a probability kernel for all $k$ and $\sum_k \alpha_k = 1$. Thus, similar to the steps in the OPE setting, for any policy $\pi$, we have
\begin{align*}
    \| Q^* - Q^{\pi_K} \|_{\rho^{\pi}}^2 &\leq  \left[ \frac{2 \gamma (1 - \gamma^{K+1})}{(1 - \gamma)^2} \right]^2  \left( \int \rho(ds) \pi(da|s) \sum_{k=0}^{K-1}\alpha_k  A_k \epsilon_k^2(s,a) + \alpha_K \right) \\ 
    &\leq \left[ \frac{2 \gamma (1 - \gamma^{K+1})}{(1 - \gamma)^2} \right]^2  \left( \int \mu(ds,da) \sum_{k=0}^{K-1} \alpha_k  \kappa_{\mu} \epsilon_k^2(s,a) + \alpha_K \right) \\ 
    &= \left[ \frac{2 \gamma (1 - \gamma^{K+1})}{(1 - \gamma)^2} \right]^2 \left(\sum_{k=0}^{K-1} \alpha_k  \kappa_{\mu}  \| \epsilon_k \|^2_{\mu} + \alpha_K \right) \\
    & \leq \frac{4 \gamma^2 \kappa_{\mu}}{(1 - \gamma)^4} \max_{0 \leq k \leq K-1} \| \epsilon_k \|_{\mu}^2 + \frac{ 4 \gamma^{K+2}}{(1 - \gamma)^3}.
\end{align*}
Thus, we have 
\begin{align*}
    \subopt(\pi^K) = \| Q^* - Q^{\pi_K} \|_{\rho^{\pi}} \leq \frac{2 \gamma \sqrt{\kappa_{\mu}}}{(1 - \gamma)^2} \max_{0 \leq k \leq K-1} \| \epsilon_k \|_{\mu} + \frac{ 2 \gamma^{K/2+1}}{(1 - \gamma)^{3/2}}.
\end{align*}
% Finally, we have 
% \begin{align*}
%     \\subopt(\pi_K) &= \mathbb{E}_{\rho} \left[ Q^*(s, \pi^*(s)) - Q^*(s, \pi_K(s)) \right] \\ 
%     &\leq \mathbb{E}_{\rho} \left[ Q^*(s, \pi^*(s))  - Q^{\pi_K}(s, \pi^*(s)) + Q^{\pi_K}(s, \pi_K(s)) - Q^*(s, \pi_K(s)) \right] \\ 
%     &\leq \| Q^* - Q^{\pi_K} \|_{\rho^{\pi^*}} + \| Q^* - Q^{\pi_K} \|_{\rho^{\pi_K}} \\ 
%     &\leq \frac{4 \gamma \sqrt{\kappa_{\mu}}}{(1 - \gamma)^2} \max_{0 \leq k \leq K-1} \| \epsilon_k \|_{\mu} + \frac{ 4 \gamma^{K/2+1}}{(1 - \gamma)^{3/2}}.
% \end{align*}
\end{proof}
\subsubsection*{Step 2: A Bellman error decomposition}
The next step of the proof is to decompose the Bellman errors $\| Q_{k+1} - T^{\pi} Q_k \|_{\mu}$ for OPE and $\| Q_{k+1} - T^* Q_k \|_{\mu}$ for OPL. Since these errors can be decomposed and bounded similarly, we only focus on OPL here. 

The difficulty in controlling the estimation error $\| Q_{k+1} - T^* Q_k \|_{2,\mu}$ is that $Q_k$ itself is a random variable that depends on the offline data $\mathcal{D}$. In particular, at any fixed $k$ with Bellman targets $\{y_i\}_{i=1}^n$ where $y_i = r_i + \gamma \max_{a'} Q_k(s_i', a')$, it is not immediate that 
$\mathbb{E}\left[ [T^* Q_k](x_i) - y_i | x_i \right] = 0$ for each covariate $x_i := (s_i, a_i)$ as $Q_k$ itself depends on $x_i$ (thus the tower law cannot apply here). A naive and simple approach to break such data dependency of $Q_k$ is to split the original data $\mathcal{D}$ into $K$ disjoint subsets and estimate each $Q_k$ using a separate subset. This naive approach is equivalent to the setting in \cite{DBLP:journals/corr/abs-1901-00137} where a fresh batch of data is generated for different iterations. This approach is however not efficient as it uses only $n/K$ samples to estimate each $Q_k$. This is problematic in high-dimensional offline RL when the number of iterations $K$ can be very large as it is often the case in practical settings. We instead prefer to use all $n$ samples to estimate each $Q_k$. This requires a different approach to handle the complicated data dependency of each $Q_k$. To circumvent this issue, we leverage a uniform convergence argument by introducing a deterministic covering of $T^* \fnn$. Each element of the deterministic covering induces a different regression target $\{r_i + \gamma \max_{a'} \tilde{Q}(s'_i, a')\}_{i=1}^n$ where $\tilde{Q}$ is a deterministic function from the covering which ensures that $\mathbb{E}\left[ r_i + \gamma \max_{a'} \tilde{Q}(s'_i, a') - [T^* \tilde{Q}](x_i) | x_i\right] = 0$. In particular, we denote 
\begin{align*}
    y_i^{Q_k} = r_i + \gamma  \max_{a'} Q_k(s'_i, a'), \forall i \text{ and } \hat{f}^{Q_k} := Q_{k+1} = \arginf_{f \in \fnn} \sum_{i=1}^n l(f(x_i), y_i^{Q_k}), \text{ and } f_*^{Q_k} = T^* Q_k, 
\end{align*}
% to emphasize the dependence on $Q_{k-1}$. Denote the empirical risk minimizer over the function class $\fnn(\mathcal{X})$ as
% \begin{align}
%     \hat{f}^{Q_{k-1}} := Q_k = \arginf_{f \in \fnn} \sum_{i=1}^n l(f(x_i), y_i^{Q_{k-1}}),
%     \label{eq:empirical_risk_minimization}
% \end{align} 
where $l(x,y) = (x-y)^2$ is the squared loss function. Note that for any deterministic $Q \in \fnn$, we have $f_*^Q(x_1) = \mathbb{E}[y_1^Q|x_1], \forall x_1$, thus 
\begin{align}
    \mathbb{E}(l_f - l_{f_*^Q}) = \| f - f_*^Q \|_{\mu}^2, \forall f, 
    \label{eq:loss_to_norm}
\end{align}
where $l_f$ denotes the random variable $(f(x_1) - y_1^{Q})^2$ for a given fixed $Q$. Now letting $f_{\perp}^Q := \arginf_{f \in \fnn} \|f - f_*^Q \|_{2, \mu}$ be the projection of $f_*^Q$ onto the function class $\fnn$, we have 
\begin{align}
    \max_{k} \| Q_{k+1} - T^*Q_k \|_{\mu}^2 &= \max_{k} \| \hat{f}^{Q_{k}} - f_*^{Q_{k}} \|_{\mu}^2 \nonumber  \overset{(a)}{\leq} \sup_{Q \in \fnn} \| \hat{f}^{Q} - f_*^{Q} \|_{\mu}^2 \overset{(b)}{=} \sup_{Q \in \fnn} \mathbb{E}(l_{\hat{f}^Q} - l_{f_*^Q}) 
    \nonumber\\
    &\overset{(c)}{\leq} \sup_{Q \in \fnn} \left\{ \mathbb{E}(l_{\hat{f}^Q} - l_{f_*^Q}) + \mathbb{E}_n(l_{f_{\perp}^Q} - l_{\hat{f}^Q}) \right\} \nonumber \\
% \end{align}
% \begin{align}
    &= \sup_{Q \in \fnn} \left\{ (\mathbb{E} - \mathbb{E}_n)(l_{\hat{f}^{Q}} - l_{f_*^{Q}}) + \mathbb{E}_n(l_{f_{\perp}^{Q}} - l_{f_*^{Q}}) \right \} \nonumber \\ 
    &\leq \underbrace{\sup_{Q \in \fnn} (\mathbb{E} - \mathbb{E}_n)(l_{\hat{f}^Q} - l_{f_*^Q})}_{I_1, \text{empirical process term}} + \underbrace{\sup_{Q \in \fnn} \mathbb{E}_n(l_{f_{\perp}^Q} - l_{f_*^Q})}_{I_2, \text{bias term}}, 
    % &= \underbrace{(\mathbb{E} - \mathbb{E}_n)(l_{\hat{f}^{Q_{k-1}}} - l_{f_*^{Q_{k-1}}})}_{\text{empirical process term}} + \underbrace{\mathbb{E}_n(l_{f_{\perp}^{Q_{k-1}}} - l_{f_*^{Q_{k-1}}})}_{\text{bias term}}
    \label{eq:decomposition}
\end{align}
where (a) follows from that $Q_{k} \in \fnn$, (b) follows from Equation (\ref{eq:loss_to_norm}), and (c) follows from that $\mathbb{E}_n [l_{\hat{f}^Q}] \leq \mathbb{E}_n [l_{f^Q}], \forall f, Q \in \fnn$. That is, the error is decomposed into two terms: the first term $I_1$ resembles the empirical process in statistical learning theory and the second term $I_2$ specifies the bias caused by the regression target $f_*^{Q}$ not being in the function space $\fnn$.

% Note that the bias term is taken uniformly over $\fnn$ instead of a fixed function in $\fnn$. 

\subsubsection*{Step 3: A deviation analysis}
The next step is to bound the empirical process term and the bias term via an intricate concentration, local Rademacher complexities and a localization argument. First, the bias term in Equation (\ref{eq:decomposition}) is taken uniformly over the function space, thus standard concentration arguments such as Bernstein's inequality and Pollard's inequality used in \cite{DBLP:journals/jmlr/MunosS08,DBLP:conf/icml/0002VY19} do not apply here. Second, local Rademacher complexities \citep{bartlett2005} are data-dependent complexity measures that exploit the fact that only a small subset of the function class will be used. Leveraging a localization argument for local Rademacher complexities \citep{farrell2018deep}, we localize an empirical Rademacher ball into smaller balls by which we can handle their complexities more effectively. Moreover, we explicitly use the sub-root function argument to derive our bound and extend the technique to the uniform convergence case. That is, reasoning over the sub-root function argument makes our proof more modular and easier to incorporate the uniform convergence argument.

% The unique fixed radius point of a sub-root the empirical Rademacher ball can be obtained explicitly via the solution to a quadratic equation. 
Localization is particularly useful to handle the complicated approximation errors induced by deep ReLU network function approximation.

\subsubsection*{Step 3.a: Bounding the bias term via a uniform convergence concentration inequality}
Before delving into our proof, we introduce relevant notations. Let $\mathcal{F} - \mathcal{G} := \{f - g: f \in \mathcal{F}, g \in \mathcal{G}\}$, let $N(\epsilon, \mathcal{F}, \| \cdot \|)$ be the $\epsilon$-covering number of $\mathcal{F}$ w.r.t. $\| \cdot \|$ norm, $H(\epsilon, \mathcal{F}, \| \cdot \|) := \log N(\epsilon, \mathcal{F}, \| \cdot \|)$ be the entropic number, 
let $N_{[]}(\epsilon, \mathcal{F}, \|\cdot\|)$ be the bracketing number of $\mathcal{F}$, i.e., the minimum number of brackets of $\|\cdot\|$-size less than or equal to $\epsilon$, necessary to cover $\mathcal{F}$, let $H_{[]}(\epsilon, \mathcal{F}, \|\cdot\|) = \log N_{[]}(\epsilon, \mathcal{F}, \|\cdot\|)$ be the $\|\cdot\|$-bracketing metric entropy of $\mathcal{F}$,let $\mathcal{F} | \{x_i\}_{i=1}^n = \{(f(x_1), ..., f(x_n)) \in \mathbb{R}^n | f \in \mathcal{F}\}$, and let $T^* \mathcal{F} = \{T^* f: f \in \mathcal{F}\}$. Finally, for sample set $\{x_i\}_{i=1}^n$, we define the empirical norm $\|f\|_n := \sqrt{\frac{1}{n}\sum_{i=1}^n f(x_i)^2}$.

We define the inherent Bellman error as  $d_{\fnn} := \sup_{Q \in \fnn} \inf_{f \in \fnn}  \| f - T^* Q\|_{\mu}$. This implies that 
\begin{align}
    d_{\fnn}^2 := \sup_{Q \in \fnn} \inf_{f \in \fnn}  \| f - T^* Q\|_{\mu}^2 = \sup_{Q \in \fnn} \mathbb{E}(l_{f_{\perp}^Q} - l_{f_*^Q}). 
\end{align}
We have 
\begin{align*}
    |l_f - l_g| \leq 4  |f-g| \text{ and } |l_f - l_g| \leq 8 . 
\end{align*}

We have 
\begin{align*}
    &H(\epsilon, \{l_{f_{\perp}^Q} - l_{f_*^Q}: Q \in \fnn\}| \{x_i, y_i\}_{i=1}^n, n^{-1} \| \cdot \|_1 )  \\
    &\leq H(\frac{\epsilon}{4}, \{f_{\perp}^Q - f_*^Q: Q \in \fnn\}| \{x_i\}_{i=1}^n, n^{-1} \| \cdot \|_1 ) \\ 
    &\leq H(\frac{\epsilon}{4 }, (\mathcal{F} - T^* \fnn)| \{x_i \}_{i=1}^n, n^{-1} \| \cdot \|_1 ) \\ 
    &\leq H(\frac{\epsilon}{8 }, \fnn| \{x_i\}_{i=1}^n, n^{-1} \| \cdot \|_1) + H(\frac{\epsilon}{8  }, T^* \fnn| \{x_i\}_{i=1}^n, n^{-1} \| \cdot \|_1 ) \\ 
    &\leq H(\frac{\epsilon}{8 }, \fnn| \{x_i\}_{i=1}^n,  \| \cdot \|_{\infty}) + H(\frac{\epsilon}{8 }, T^* \fnn,  \| \cdot \|_{\infty})
\end{align*}

For any $\epsilon' > 0$ and $\delta' \in (0,1)$, it follows from Lemma \ref{lemma:sup_concentration} with $\epsilon = 1/2$ and $\alpha = \epsilon'^2$, with probability at least $1 - \delta'$, for any $Q \in \fnn$, we have 
\begin{align}
    \mathbb{E}_n(l_{f_{\perp}^Q} - l_{f_*^Q}) \leq 3 \mathbb{E}(l_{f_{\perp}^Q} - l_{f_*^Q}) + \epsilon'^2 \leq 3 d_{\fnn}^{2} + \epsilon'^2, 
    \label{eq:bound_E_minus_En_lf_minus_lfstar}
\end{align}
given that 
\begin{align*}
    n \approx \frac{1}{\epsilon'^2}\left( \log(4/\delta') + \log \mathbb{E} N(\frac{\epsilon'^2}{40 }, (\fnn - T^* \fnn)| \{x_i\}_{i=1}^n, n^{-1} \| \cdot \|_1) \right) .
\end{align*}

Note that if we use Pollard's inequality \citep{DBLP:journals/jmlr/MunosS08} in the place of Lemma \ref{lemma:sup_concentration}, the RHS of Equation (\ref{eq:bound_E_minus_En_lf_minus_lfstar}) is bounded by $\epsilon'$ instead of $\epsilon'^2$(i.e., $n$ scales with $O(1/\epsilon'^4)$ instead of $O(1/\epsilon'^2)$). In addition, unlike \cite{DBLP:conf/icml/0002VY19}, the uniform convergence argument hinders the application of Bernstein's inequality. We remark that \citealt{DBLP:conf/icml/0002VY19} makes a mistake in their proof by ignoring the data-dependent structure in the algorithm (i.e., they wrongly assume that $Q^k$ in Algorithm \ref{alg:LSVI} is fixed and independent of $\{s_i, a_i\}_{i=1}^n$). Thus, the uniform convergence argument in our proof is necessary. 

\subsubsection*{Step 3.b: Bounding the empirical process term via local Rademacher complexities}
For any $Q \in \fnn$, we have 
\begin{align*}
    |l_{f_{\perp}^Q} - l_{f_*^Q}| &\leq 2 |f_{\perp}^Q - f_*^Q| \leq 2, \\ \mathbb{V}[l_{f_{\perp}^Q} - l_{f_*^Q}] &\leq \mathbb{E}[(l_{f_{\perp}^Q} - l_{f_*^Q})^2] \leq 4 \mathbb{E} (f_{\perp}^Q - f_*^Q)^2. 
 \end{align*}
Thus, it follows from Lemma \ref{lemma:local_rademacher_complexity_basics_first_part} (with $\alpha = 1/2$) that with any $r > 0, \delta \in (0,1)$, with probability at least $1 - \delta$, we have
\begin{align*}
    &\sup \{(\mathbb{E} - \mathbb{E}_n)(l_{\hat{f}^Q} - l_{f_*^Q}): Q \in \fnn, \|\hat{f}^Q - f_*^Q\|^2_{\mu} \leq r \} \\
    &\leq \sup \{(\mathbb{E} - \mathbb{E}_n)(l_f - l_g): f \in \fnn, g \in T^{*} \mathcal{F}, \|f - g\|^2_{\mu} \leq r \} \\
     &\leq 3 \mathbb{E} R_n \left\{l_f - l_g: f \in \fnn, g \in T^{*} \fnn, \|f - g \|_{\mu}^2 \leq r \right \} + 2 \sqrt{ \frac{2 r \log(1/\delta) }{n} } + \frac{28 \log(1/\delta)}{3n} \nonumber \\ 
     &\leq 6 \mathbb{E} R_n \left\{f - g: f \in \fnn, g \in T^{*} \fnn, \|f - g \|_{\mu}^2 \leq r \right \} + 2 \sqrt{ \frac{2 r \log(1/\delta) }{n} } + \frac{28 \log(1/\delta)}{3n}. 
\end{align*}

\subsubsection*{Step 3.c: Bounding $\|Q_{k+1} - T^* Q_k\|_{\mu}$ using localization argument via sub-root functions}
We bound $\|Q_{k+1} - T^* Q_k \|_{\mu}$ using the localization argument, breaking down the Rademacher complexities into local balls and then build up the original function space from the local balls. Let $\psi$ be a sub-root function  \citep[Definition~3.1]{bartlett2005} with the fixed point $r_*$ and assume that for any $ r \geq r_*$, we have 
\begin{align}
    \psi(r) \geq  3 \mathbb{E} R_n \left\{f - g: f \in \fnn, g \in T^{*} \fnn, \|f - g \|_{\mu}^2 \leq r \right \}.
    \label{eq:sub_root_function_bounding_empirical_Rademacher}
\end{align}

We recall that a function $\psi: [0, \infty) \rightarrow [0, \infty)$ is \textit{sub-root} if it is non-negative, non-decreasing and $r \mapsto \psi(r) / \sqrt{r}$ is non-increasing for $r > 0$. Consequently, a sub-root function $\psi$ has a unique fixed point $r_*$ where $r_* = \psi(r_*)$. In addition, $\psi(r) \leq \sqrt{r r_*}, \forall r \geq r_*$. In the next step, we will find a sub-root function $\psi$ that satisfies the inequality above, but for this step we just assume that we have such $\psi$ at hand. Combining Equations (\ref{eq:decomposition}), (\ref{eq:bound_E_minus_En_lf_minus_lfstar}), and (\ref{eq:sub_root_function_bounding_empirical_Rademacher}), we have: for any $r \geq r_*$ and any $\delta \in (0,1)$, if $\| \hat{f}^{Q_{k-1}} - f_*^{Q_{k-1}} \|_{2,\mu}^2 \leq r$, with probability at least $1 - \delta$, 
\begin{align*}
    \| \hat{f}^{Q_{k-1}} - f_*^{Q_{k-1}} \|_{2,\mu}^2 &\leq 2\psi(r) +  2 \sqrt{ \frac{2 r \log(2/\delta) }{n} } + \frac{28 \log(2/\delta)}{3n} + 3 d^2_{\mathcal{F}} + \epsilon'^2 \\ 
    &\leq \sqrt{r r_*} + 2 \sqrt{ \frac{2 r \log(2/\delta)}{n} } + \frac{28 \log(2/\delta)}{3n} + (\sqrt{3} d_{\mathcal{F}} + \epsilon' )^2, 
\end{align*} 
where 
\begin{align*}
    n \approx \frac{1}{4\epsilon'^2}\left( \log(8/\delta) + \log \mathbb{E} N(\frac{\epsilon'^2}{20}, (\fnn - T^* \fnn)| \{x_i\}_{i=1}^n, n^{-1}\| \cdot \|_1) \right).
\end{align*}

Consider $r_0 \geq r_*$ (to be chosen later) and denote the events 
\begin{align*}
    B_k := \{ \| \hat{f}^{Q_{k-1}} - f_*^{Q_{k-1}} \|^2_{2,\mu} \leq 2^k r_0   \}, \forall k \in \{0,1,...,l\},
\end{align*}
where $l = \log_2(\frac{1}{r_0}) \leq \log_2( \frac{1}{r_*} )$. We have $B_0 \subseteq B_1 \subseteq ... \subseteq B_l$ and since $\|f - g\|_{\mu}^2 \leq 1, \forall |f|_{\infty}, |g|_{\infty} \leq 1$, we have $P(B_l) = 1$. 
If $\|\hat{f}^{Q_{k-1}} - f_*^{Q_{k-1}} \|^2_{\mu} \leq 2^i r_0$ for some $i \leq l$, then with probability at least $1 - \delta$, we have
\begin{align*}
    \| \hat{f}^{Q_{k-1}} - f_*^{Q_{k-1}} \|_{2,\mu}^2 
    &\leq \sqrt{2^i r_0 r_*} + 2 \sqrt{ \frac{2^{i+1} r_0 \log(2/\delta)}{n} } + \frac{28 \log(2/\delta)}{3n} + (\sqrt{3} d_{\fnn} + \epsilon' )^2 \\
    &\leq 2^{i-1} r_0, 
\end{align*} 
if the following inequalities hold 
\begin{align*}
    \sqrt{ 2^i r_*} + 2\sqrt{ \frac{2^{i+1} \log(2/\delta)}{n} } &\leq \frac{1}{2} 2^{i-1} \sqrt{r_0},  \\ 
    \frac{28 \log(2/\delta)}{3n} + (\sqrt{3} d_{\fnn} + \epsilon' )^2 &\leq \frac{1}{2} 2^{i-1} r_0. 
\end{align*}

We choose $r_0 \geq r_*$ such that the inequalities above hold for all $0 \leq i \leq l$. This can be done by simply setting
\begin{align*}
    \sqrt{r_0} &= \frac{2}{2^{i-1}}  \left( \sqrt{ 2^i r_*} + 2\sqrt{ \frac{2^{i+1} \log(2/\delta) }{n} } \right) \bigg |_{i=0} + \sqrt{  \frac{2}{2^{i-1}} \left(  \frac{28 \log(2/\delta)}{3n} + (\sqrt{3} d_{\fnn} + \epsilon' )^2  \right)   } \bigg |_{i=0} \\
    % &= (96 M + 1) \sqrt{r_*} + \sqrt{ \frac{8 \lambda}{n} } + 2 \sqrt{ \frac{32M \lambda}{n} + (d_{\mathcal{F}} + 2 \sqrt{ \frac{2 \lambda}{n} })^2   } \\
    &\lesssim  d_{\fnn} + \epsilon' + \sqrt{\frac{\log(2/\delta)}{n}} + \sqrt{r_*}.
\end{align*}

% Since $\hat{f} \in \mathcal{F}$, there exists some $ 0 \leq k \leq l$ such that $P(A_k) = 1$. 
Since $\{B_i\}$ is a sequence of increasing events, we have 
\begin{align*}
    P(B_0) &= P(B_1) - P(B_1 \cap B_0^c ) = P(B_2) - P(B_2 \cap B_1^c) - P(B_1 \cap B_0^c) \\ 
    &=P(B_l) - \sum_{i=0}^{l-1} P(B_{i+1} \cap B_i^c) \geq 1 - l \delta.
\end{align*}
Thus, with probability at least $1 - \delta$, we have
% \begin{align}
%     \| \hat{f}^{Q_{k-1}} - f_*^{Q_{k-1}} \|_{2,\mu} &\lesssim  d_{\fnn} + \epsilon' + \sqrt{\frac{\log(2l/\delta)}{n}} + \sqrt{r_*} ,
%     \label{eq:general_bounding}
% \end{align}
\begin{equation}
\| \hat{f}^{Q_{k-1}} - f_*^{Q_{k-1}} \|_{\mu} \lesssim  d_{\fnn} + \epsilon' + \sqrt{\frac{\log(2l/\delta)}{n}} + \sqrt{r_*}
\label{eq:general_bounding}
\end{equation}
where 
\begin{align*}
    n \approx \frac{1}{4\epsilon'^2}\left( \log(8l/\delta) + \log \mathbb{E} N(\frac{\epsilon'^2}{20}, (\fnn - T^* \fnn)| \{x_i\}_{i=1}^n, n^{-1}\| \cdot \|_1)) \right) .
\end{align*}

\subsubsection*{Step 3.d: Finding a sub-root function and its fixed point}
% \noindent \textbf{Finding a sub-root function and its fixed point}. 
It remains to find a sub-root function $\psi(r)$ that satisfies Equation (\ref{eq:sub_root_function_bounding_empirical_Rademacher}) and thus its fixed point. The main idea is to bound the RHS, the local Rademacher complexity, of Equation (\ref{eq:sub_root_function_bounding_empirical_Rademacher}) by its empirical counterpart as the latter can then be further bounded by a sub-root function represented by a measure of compactness of the function spaces $\fnn$ and $T^{*}\fnn$. 

% the local Rademacher complexity $\mathbb{E} R_n \{f- g: f \in \mathcal{F}, g \in T^{\pi} \mathcal{F}, \|f - g \|_{\mu}^2 \leq r \}$ by its empirical local Rademacher complexity $\mathbb{E}_{\sigma} R_n \{f- g: f \in \mathcal{F}, g \in T^{\pi} \mathcal{F}, \|f - g \|_n \leq r' \}$ for some $r' = r'(r)$ depending on $r$ as the latter can then be bounded by a sub-root function represented by covering numbers and pseudo-dimension of $\mathcal{F}$ and $T^{\pi} \mathcal{F}$. 

% It follows from Lemma 
% %5 
% \ref{lemma:local_empirical_rademacher_bounded_by_covering_number_with_empirical_norm} 
% and Lemma 
% %8 
% \ref{lemma:data_dependent_covering_number_bounded_by_pseudo_dimension}
% (Appendix B) that 

For any $\epsilon > 0$, we have the following inequalities for entropic numbers:
\begin{align}
H(\epsilon, \fnn - T^* \fnn, \| \cdot \|_{n}) &\leq H(\epsilon/2, \fnn, \| \cdot \|_{n}) + H(\epsilon/2,  T^* \fnn, \| \cdot \|_{n}), \nonumber \\
H(\epsilon, \fnn, \|\cdot\|_{n}) &\leq 
H(\epsilon, \fnn| \{x_i\}_{i=1}^n, \| \cdot \|_{\infty}) \overset{(a)}{\lesssim} N[ (\log N)^2 + \log(1/\epsilon)],  \\
% \label{eq:explicit_entropic_number_F}
H(\epsilon, T^* \fnn, \|\cdot \|_{n}) &\leq H(\epsilon, T^* \fnn, \| \cdot \|_{\infty}) \leq H_{[]}(2 \epsilon, T^*\fnn, \| \cdot \|_{\infty}) \nonumber \\
&\overset{(b)}{\leq} H_{[]}(2\epsilon, \bar{B}^{\alpha}_{p,q}(\mathcal{X}), \| \cdot \|_{\infty}) \overset{(c)}{\lesssim} (2\epsilon)^{-d/\alpha},
% \label{eq:explicit_entropic_number_TF}
\end{align}
where $N$ is a hyperparameter of the deep ReLU network described in Lemma \ref{lemma:approximation_power_for_Besov}, (a) follows from Lemma \ref{lemma:approximation_power_for_Besov}, and (b) follows from  Assumption \ref{assumption:completeness}, and (c) follows from Lemma \ref{lemma:entropic_number_of_Besov}. Let $\mathcal{H} := \fnn - T^{*} \fnn$, it follows from Lemma \ref{lemma:refined_entropy_integral} with $\{\xi_k := \epsilon / 2^k\}_{k \in \mathbb{N}}$ for any $\epsilon > 0$ that 
\begin{align*}
    &\mathbb{E}_{\sigma} R_n \{ h \in \mathcal{H} - \mathcal{H}: \|h\|_{n} \leq \epsilon \} \leq 4 \sum_{k=1}^{\infty} \frac{\epsilon}{2^{k-1}} \sqrt{ \frac{H(\epsilon/2^{k-1}, \mathcal{H}, \| \cdot \|_{n})}{n}  } \nonumber \\ 
    &\leq  4 \sum_{k=1}^{\infty} \frac{\epsilon}{2^{k-1}} \sqrt{ \frac{H(\epsilon/2^{k}, \fnn, \| \cdot \|_{\infty})}{n}  } + 4 \sum_{k=1}^{\infty} \frac{\epsilon}{2^{k-1}} \sqrt{ \frac{H(\epsilon/2^{k}, T^{\pi} \fnn, \| \cdot \|_{\infty})}{n}  }  \\ 
    &\leq \frac{4 \epsilon}{\sqrt{n}} \sum_{k=1}^{\infty} 2^{-(k-1)} \sqrt{N \left( (\log N)^2 + \log(2^k/\epsilon) \right)} + \frac{4 \epsilon}{\sqrt{n}} \sum_{k=1}^{\infty} 2^{-(k-1)} \sqrt{\left(\frac{\epsilon}{2^{k-1}} \right)^{-d/\alpha}} \\ 
    &\lesssim \frac{\epsilon}{\sqrt{n}} \sqrt{N ((\log N)^2 + \log(1/\epsilon))} + \frac{\epsilon^{1 - \frac{d}{2 \alpha}}}{\sqrt{n}},
    % &\lesssim 4 \sqrt{ \frac{Pdim(\fnn)}{n} } \sum_{k=1}^N \frac{\epsilon}{2^{k-1}} \sqrt{\log(\beta_n(\fnn)/\epsilon) + k \log2} + \frac{4 \epsilon}{\sqrt{n}} \sum_{k=1}^N 2^{-(k-1)} \sqrt{ (\frac{\epsilon}{2^{k-1}})^{-d/\alpha}  } + \frac{\epsilon}{2^N}. 
    % &\leq 4 \sum_{k=1}^N \xi_{k-1} \sqrt{ \frac{\log \mathcal{N}(\xi_k/4, \mathcal{F}, L_2(P_n))}{n}  } + 4 \sum_{k=1}^N \xi_{k-1} \sqrt{ \frac{\log \mathcal{N}(\xi_k/4, T^{\pi}\mathcal{F}, L_2(P_n))}{n}  } + \xi_N \\ 
    % &\leq 4 \sqrt{\frac{Pdim(\mathcal{H})}{n}} \sum_{i=1}^N \frac{\epsilon}{2^{k-1}} \sqrt{ (k+1) \log 2 + \log(\beta_n(\mathcal{H}) \epsilon^{-1})    } 
    % \\
    % &+ 4 \sqrt{\frac{Pdim(\mathcal{T^{\pi} F})}{n}} \sum_{i=1}^N \frac{\epsilon}{2^{k-1}} \sqrt{ (k+2) \log 2 + \log(\beta_n(\mathcal{T^{\pi} F}) \epsilon^{-1})    } 
    % + \frac{\epsilon}{2^N}. 
\end{align*}
where we use $\sqrt{a + b} \leq \sqrt{a} + \sqrt{b}, \forall a,b \geq 0$, $\sum_{k=1}^{\infty} \frac{\sqrt{k}}{2^{k-1}} < \infty$, and $\sum_{k=1}^{\infty} \left( \frac{1}{2^{1 - \frac{d}{2 \alpha}}} \right)^{k-1} < \infty$.

% Here, the first inequality follows from Lemma 
% %6
% \ref{lemma:refined_entropy_integral} 
% (Appendix B), the second inequality follows from Lemma 
% %4 
% \ref{lemma:covering_number_of_symmetric_function_set_bounded_by_that_of_original_function_set} 
% (Appendix B), the third inequality follows from (\ref{eq:from_F_start_to_Pdim_F}), and the final inequality follows from that $\sqrt{a + b} \leq \sqrt{a} + \sqrt{b}, \forall a,b \geq 0$. Now, take $N \rightarrow \infty$ both sides of (\ref{eq:local_empirical_rademacher_bounded_by_geometric_series}) and choose $\epsilon  \leq  \beta_n(\mathcal{H}) / 2$, noting that $\sum_{k=1}^{\infty} \frac{\sqrt{k+1}}{2^{k-1}} < \infty$, we have 
% \begin{align*}
% \mathbb{E}_{\sigma} R_n \{ h \in \mathcal{H} - \mathcal{H}: \|h\|_n \leq \epsilon \} \leq C_2  \frac{\epsilon}{\sqrt{n}} \left( \sqrt{Pdim(\mathcal{F}) \log(\beta_n(\mathcal{F}) / \epsilon)}  \right),
% \end{align*}
% for some absolute constant $C_2 > 0$. Combining the inequality above with Lemma 
%5
It now follows from Lemma
\ref{lemma:local_empirical_rademacher_bounded_by_covering_number_with_empirical_norm} that 
\begin{align*}
&\mathbb{E}_{\sigma} R_n \{f \in \mathcal{F}, g \in T^{*} \mathcal{F}: \| f -g \|_{n}^2 \leq r \} \\
&\leq \inf_{\epsilon > 0} \bigg[ \mathbb{E}_{\sigma} R_n \{ h \in \mathcal{H} - \mathcal{H}: \| h \|_{\mu} \leq \epsilon \}  + \sqrt{ \frac{2r H(\epsilon/2, \mathcal{H}, \|\cdot \|_{n}) }{n}  } \bigg] \\ 
&\lesssim  \bigg[ \frac{\epsilon}{\sqrt{n}} \sqrt{N ((\log N)^2 + \log(1/\epsilon))} + \frac{\epsilon^{1 - \frac{d}{2 \alpha}}}{\sqrt{n}} + 
\sqrt{\frac{2r}{n}} \sqrt{N((\log N)^2 + \log(4 /\epsilon))} + \sqrt{\frac{2r}{n}}(\epsilon/2)^{\frac{-d}{2 \alpha}}
\bigg] \bigg |_{\epsilon = n^{-\beta}}\\ 
&\asymp n^{-\beta -1/2} \sqrt{N (\log^2 N + \log n )} + n^{-\beta(1 - \frac{d}{2 \alpha}) - 1/2} + \sqrt{\frac{r}{n}}  \sqrt{N (\log^2 N + \log n )} + \sqrt{r} n^{-\frac{1}{2}(1 - \frac{\beta d}{\alpha})} =: \psi_1(r),
% &\lesssim n^{-1} \sqrt{N( (\log N)^2 + \log(\sqrt{n}))} + \frac{1}{n^{1 - \frac{d}{4 \alpha}}} + \sqrt{\frac{r}{n}} \sqrt{N( (\log N)^2 + \log(4 \sqrt{n}))} + \frac{\sqrt{r}}{n^{\frac{1}{2}(1 - \frac{d}{2 \alpha})}} =: \psi_1(r),\\
% &\lesssim \frac{ \sqrt{ Pdim(\fnn) \log \frac{ n \sqrt{n}}{ Pdim(\fnn) }  }  }{n}  + \frac{1}{n^{1 - \frac{d}{4 \alpha}}} + \frac{  \sqrt{r Pdim(\fnn)  \log \frac{ n \sqrt{n}}{ Pdim(\fnn)} }}  {n^{1/2} }    + \frac{\sqrt{r}}{n^{\frac{1}{2}(1 - \frac{d}{2 \alpha})}} \\
%  &=: \psi_1(r),
\end{align*}
where $\beta \in (0, \frac{\alpha}{d})$ is an absolute constant to be chosen later.

% \lor \left(\frac{Pdim(\fnn)}{e V_{\max}} \right)^{2/3} $. 
% Now, we bound $ \mathbb{E}_{\sigma} R_n \{f - g: f \in \fnn, g \in T^{\pi} \fnn, \| f - g \|_{\mu}^2 \leq r \}$ by $ \mathbb{E}_{\sigma} R_n \{f \in \mathcal{F}_{}, g \in T^{\pi} \mathcal{F}: \| f -g \|_n^2 \leq r' \}$ for some $r' = r'(r)$.
Note that $\mathbb{V}[(f-g)^2] \leq \mathbb{E}[(f-g)^4] \leq \mathbb{E}[(f-g)^2]$ for any $f \in \fnn, g \in T^* \fnn$. Thus, for any $r \geq r_*$, it follows from Lemma \ref{lemma:local_rademacher_complexity_basics} that with probability at least $1 - \frac{1}{n}$, we have the following inequality for any $f \in \fnn, g \in T^* \fnn$ such that $\|f - g\|^2_{\mu} \leq r$, 
\begin{align*}
&\|f - g\|_{n}^2 \\
&\leq \|f - g \|^2_{\mu} + 3 \mathbb{E} R_n \{(f-g)^2: f \in \fnn, g \in T^* \fnn, \| f - g \|^2_{\mu} \leq r\} +  \sqrt{ \frac{2r \log n}{n} } + \frac{56}{3} \frac{\log n}{n} \\ 
&\leq  \|f - g \|^2_{\mu} + 3 \mathbb{E} R_n \{f-g: f \in \fnn, g \in T^* \fnn, \| f - g \|^2_{\mu} \leq r\} + \sqrt{ \frac{2r \log n }{n} } + \frac{56}{3} \frac{\log n}{n} \\ 
&\leq r + \psi(r) + r + r \leq 4r,
\end{align*}
if $r \geq r_* \lor \frac{2 log n}{n} \lor \frac{56 log n}{3 n}$. For such $r$, denote $E_r = \{ \| f - g \|_{n}^2 \leq 4r \} \cap \{\|f - f_*\|_{\mu}^2 \leq r \}$, we have $P(E_r) \geq 1 - 1/n$ and
\begin{align*}
&3 \mathbb{E} R_n \{f - g: f \in \fnn, g \in T^* \fnn, \|f - g\|^2_{\mu} \leq r \} \\
&=
3 \mathbb{E} \mathbb{E}_{\sigma} R_n \{f - g: f \in \fnn, g \in T^{*} \fnn, \|f - g\|^2_{\mu} \leq r \} \\
&\leq 3 \mathbb{E} \bigg[ 1_{E_r} \mathbb{E}_{\sigma} R_n \{f - g: f \in \fnn, g \in T^* \fnn, \|f - g\|^2_{\mu} \leq r \} +  (1 - 1_{E_r})  \bigg] \\ 
&\leq 3 \mathbb{E} \bigg[ \mathbb{E}_{\sigma} R_n \{f - g: f \in \fnn, g \in T^* \fnn, \|f - g\|^2_{n} \leq 4r \} +  (1 - 1_{E_r})  \bigg] \\
&\leq 3( \psi_1(4r) + \frac{1}{n}) \\ 
&\lesssim n^{-\beta -1/2} \sqrt{N (\log^2 N + \log n )} + n^{-\beta(1 - \frac{d}{2 \alpha}) - 1/2} + \sqrt{\frac{r}{n}}  \sqrt{N (\log^2 N + \log n )} \\
&+ \sqrt{r} n^{-\frac{1}{2}(1 - \frac{\beta d}{\alpha})} + n^{-1} =: \psi(r).
\end{align*}
It is easy to verify that $\psi(r)$ defined above is a sub-root function. The fixed point $r_*$ of $\psi(r)$ can be solved analytically via the simple quadratic equation $r_* = \psi(r_*)$. In particular, we have 
\begin{align}
    \sqrt{r_*} &\lesssim n^{-1/2} \sqrt{N (\log^2 N + \log n )} + n^{-\frac{1}{2}(1 - \frac{\beta d}{\alpha})} + n^{-\frac{\beta}{2} - \frac{1}{4}} [N (\log^2 N + \log n )]^{1/4} \nonumber \\
    &+ n^{-\frac{\beta}{2}(1 - \frac{d}{2 \alpha}) - \frac{1}{2}} + n^{-1/2} \nonumber \\ 
    &\lesssim n^{-\frac{1}{4} ( (2\beta) \land 1) + 1)} \sqrt{N (\log^2 N + \log n )} + n^{-\frac{1}{2}(1 - \frac{\beta d}{\alpha})} + n^{-\frac{\beta}{2}(1 - \frac{d}{2 \alpha}) - \frac{1}{2}} + n^{-1/2}.
    \label{eq:sandwich_fixed_point}
\end{align}
It follows from Equation (\ref{eq:general_bounding}) (where $l \lesssim \log(1/r_*)$), the definition of $d_{\fnn}$, Lemma \ref{lemma:approximation_power_for_Besov}, and Equation (\ref{eq:sandwich_fixed_point}) that for any $\epsilon' >0$ and $\delta \in (0,1)$, with probability at least $1 - \delta$, we have 
\begin{align}
    \max_{k} \|Q_{k+1} - T^* Q_k \|_{\mu} &\lesssim N^{-\alpha / d} + \epsilon' + n^{-\frac{1}{4} ( (2\beta) \land 1) + 1)} \sqrt{N (\log^2 N + \log n )} + n^{-\frac{1}{2}(1 - \frac{\beta d}{\alpha})} \nonumber \\
    &+ n^{-\frac{\beta}{2}(1 - \frac{d}{2 \alpha}) - \frac{1}{2}}
    + n^{-1/2}\sqrt{\log(1/\delta) + \log \log n  },
    \label{eq:final_parametric_upper_bound}
\end{align}
where 
\begin{align}
    n &\gtrsim \frac{1}{4\epsilon'^2}\bigg( \log(1/\delta) + \log \log n + \log \mathbb{E} N(\frac{\epsilon'^2}{20}, (\fnn - T^* \fnn)| \{x_i\}_{i=1}^n, n^{-1} \cdot \| \cdot \|_1)) \bigg) .
    \label{eq:sample_complexity_n}
\end{align}

% As a result, we finally obtain that for any $\epsilon' >0$ and $\delta \in (0,1)$, with probability at least $1 - \delta$, we have 
% \begin{align*}
%     \max_{k} \| Q_{k+1} - T^*Q_k \|_{2, \mu}^2 &\lesssim N^{-\alpha / d} + \epsilon' + n^{-\frac{1}{4} ( (2\beta) \land 1) + 1)} \sqrt{N (\log^2 N + \log n )} + n^{-\frac{1}{2}(1 - \frac{\beta d}{\alpha})} \nonumber \\
%     &+ n^{-\frac{\beta}{2}(1 - \frac{d}{2 \alpha}) - \frac{1}{2}}
%     + n^{-1/2}\sqrt{\log(1/\delta) + \log \log n  } 
%     % \label{eq:final_parametric_upper_bound}
% \end{align*}
% where $\beta \in (0, \frac{\alpha}{d}), N \in \mathbb{N}$, and
% \begin{align*}
%     n &\gtrsim \frac{1}{\epsilon'^2}\bigg( \log(1/\delta) + \log \log n + \log \mathbb{E} N(\frac{\epsilon'^2}{40 }, (\fnn - T^* \fnn)| \{x_i\}_{i=1}^n, n^{-1} \cdot \| \cdot \|_1)) \bigg) .
%     % \label{eq:sample_complexity_n}
% \end{align*}

\subsubsection*{Step 4: Minimizing the upper bound}
The final step for the proof is to minimize the upper error bound obtained in the previous steps w.r.t. two free parameters $\beta \in (0, \frac{\alpha}{d})$ and $N \in \mathbb{N}$. Note that $N$ parameterizes the deep ReLU architecture $\Phi(L,m,S,B)$ given Lemma \ref{lemma:approximation_power_for_Besov}. % This process leads to $N \asymp n^{\frac{(\beta + 1/2)d}{2 \alpha + d}}$, and $\beta = (2 + \frac{d^2}{\alpha(\alpha + d)})^{-1}$ for the optimal upper error bound up to log factors. 
In particular, we optimize over $\beta \in (0, \frac{\alpha}{d})$ and $N \in \mathbb{N}$ to minimize the upper bound in the RHS of Equation (\ref{eq:final_parametric_upper_bound}). The RHS of Equation (\ref{eq:final_parametric_upper_bound}) is minimized (up to $\log n$-factor) by choosing
\begin{align}
    N \asymp n^{\frac{1}{2}((2 \beta \land 1) + 1) \frac{d}{2\alpha + d}} \text{ and }
    \beta = \left(2 + \frac{d^2}{\alpha (\alpha + d)} \right)^{-1}, 
    \label{eq:optimal_value_N_beta}
\end{align}
which results in $N \asymp n^{\frac{1}{2}(2 \beta  + 1) \frac{d}{2\alpha + d}}$. At these optimal values, Equation (\ref{eq:final_parametric_upper_bound}) becomes 
\begin{align}
    \max_{k} \|Q_{k+1} - T^* Q_k\|_{\mu} &\lesssim \epsilon' + n^{-\frac{1}{2}\left( \frac{2 \alpha}{2 \alpha + d} + \frac{d}{ \alpha} \right)^{-1}} \log n + n^{-1/2}\sqrt{\log(1/\delta) + \log \log n  },
    \label{eq:final_upper_bound}
\end{align}
where we use inequalities $n^{-\frac{\beta}{2}(1 - \frac{d}{2 \alpha}) - \frac{1}{2}} \leq n^{-\frac{1}{2}(1 - \frac{\beta d}{\alpha})} \asymp N^{-\alpha /d} = n^{-\frac{1}{2}\left( \frac{2 \alpha}{2 \alpha + d} + \frac{d}{ \alpha} \right)^{-1}} $.

Now, for any $\epsilon > 0$, we set $\epsilon' = \epsilon/3$ and let 
\begin{align*}
    n^{-\frac{1}{2}\left( \frac{2 \alpha}{2 \alpha + d} + \frac{d}{ \alpha} \right)^{-1}} \log n \lesssim \epsilon / 3 \text{ and } n^{-1/2}\sqrt{\log(1/\delta) + \log \log n  } \lesssim \epsilon / 3.
\end{align*}
It then follows from Equation (\ref{eq:final_upper_bound}) that with probability at least $1 - \delta$, we have $\max_{k} \| Q_{k+1} - T^* Q_k \|_{\mu} \leq \epsilon$ if $n$ simultaneously satisfies Equation (\ref{eq:sample_complexity_n}) with $\epsilon' = \epsilon/3$ and 
\begin{align}
    n \gtrsim \left(\frac{1}{\epsilon^2} \right)^{ \frac{2 \alpha}{ 2 \alpha + d} + \frac{d}{\alpha}} (\log^2 n)^{ \frac{2 \alpha}{ 2 \alpha + d}+ \frac{d}{\alpha}} \text{ and } n \gtrsim \frac{1}{\epsilon^2} \left( \log(1/\delta) + \log \log n  \right).
    \label{eq:explicit_n_1}
\end{align}

Next, we derive an explicit formula of the sample complexity satisfying Equation (\ref{eq:sample_complexity_n}). Using Equations (\ref{eq:final_parametric_upper_bound}), (\ref{eq:explicit_n_1}), and (\ref{eq:optimal_value_N_beta}), we have that $n$ satisfies Equation (\ref{eq:sample_complexity_n}) if 
\begin{align}
\begin{cases}
        n &\gtrsim \frac{1}{\epsilon^2} \left[ n^{\frac{2\beta + 1}{2} \frac{d}{2 \alpha + d}} (\log^2 n + \log(1/\epsilon))\right], \\ 
    n &\gtrsim \left( \frac{1}{\epsilon^2} \right)^{1 + \frac{d}{\alpha}}, \\
    n &\gtrsim \frac{1}{\epsilon^2} \left( \log(1/\delta) + \log \log n  \right).
\end{cases}
\label{eq:explicit_n_2}
\end{align}

Note that $\beta \leq 1/2$ and $\frac{d}{\alpha} \leq 2$; thus, we have 
\begin{align*}
    \left(1 - \frac{2\beta + 1}{2} \frac{d}{2 \alpha + d} \right)^{-1} \leq 1 + \frac{d}{\alpha} \leq 3.
\end{align*}
Hence, $n$ satisfies Equations (\ref{eq:explicit_n_1}) and (\ref{eq:explicit_n_2}) if 
\begin{align*}
    n \gtrsim \left(\frac{1}{\epsilon^2} \right)^{1 + \frac{d}{\alpha}} \log^6 n + \frac{1}{\epsilon^2}(\log(1/\delta) + \log \log n).
\end{align*}

%% file: body/appendixB.tex
\renewcommand{\thesection}{B}
\section{Technical Lemmas}
\label{appendix:B}

\begin{lem}[\citet{bartlett2005}]
Let $r > 0$ and let 
\begin{align*}
    \mathcal{F} \subseteq \{f: \mathcal{X} \rightarrow [a,b] : \mathbb{V}[f(X_1)] \leq r\}. 
\end{align*}
\begin{enumerate}
    \item 
    For any $\lambda > 0$, we have with probability at least $1 - e^{-\lambda}$, 
    \begin{align*}
        \sup_{f \in \mathcal{F}} \left(\mathbb{E}f - \mathbb{E}_n f \right) 
        \leq \inf_{\alpha >0 } \left( 2(1+\alpha) \mathbb{E} \left[R_n \mathcal{F} \right] + \sqrt{ \frac{2 r \lambda}{ n}} + (b-a)\left(\frac{1}{3} + \frac{1}{\alpha} \right)\frac{\lambda}{n} \right).
    \end{align*}
    \label{lemma:local_rademacher_complexity_basics_first_part}

    \item 
    With probability at least $1 - 2 e^{-\lambda}$, 
    \begin{align*}
        \sup_{f \in \mathcal{F}} \left(\mathbb{E}f - \mathbb{E}_n f \right) 
        \leq \inf_{\alpha \in (0,1) } \left( \frac{2(1+\alpha)}{(1 - \alpha)} \mathbb{E}_{\sigma} \left[R_n \mathcal{F} \right] + \sqrt{ \frac{2 r \lambda}{ n}} + (b-a)\left(\frac{1}{3} + \frac{1}{\alpha} + \frac{1 + \alpha}{2 \alpha (1 - \alpha)} \right)\frac{\lambda}{n} \right).
    \end{align*}
\end{enumerate}
Moreover, the same results hold for $\sup_{f \in \mathcal{F}} \left(  \mathbb{E}_n f - \mathbb{E}f\right) $. 

\label{lemma:local_rademacher_complexity_basics}
\end{lem}

% \noindent \textbf{Remark}. The second inequality above bounds the suprema of the empirical process by the empirical Rademacher average instead of the population Rademacher average as in the first inequality. This is beneficial as the empirical Rademacher average can be computed in many cases. In addition, the presence of the variance in the upper in this result is a key to obtain a better bound. 

% \begin{lem}[Local Rademacher ]
% Let $\mathcal{F} \subseteq \{f: \mathcal{X} \rightarrow [-b,b]\}$ for some $b > 0$. For every $\lambda > 0$ and $r$ satisfy 
% \begin{align*}
%     r \geq 10 b \mathbb{E}R_n \{f: f \in \mathcal{F}: \|f \|_{P}^2 \leq r \} + \frac{11 b^2 \lambda}{n},
% \end{align*}
% with probability at least $1 - e^{-\lambda}$, 
% \begin{align*}
%     \{ f \in \mathcal{F}: \| f \|_P^2 \leq r \} \subseteq \{ f \in \mathcal{F}: \|f \|_{P_n}^2 \leq 2 r\}. 
% \end{align*}
% \end{lem}

\begin{lem}[{\citet[Theorem~11.6]{DBLP:books/daglib/0035701}}]
Let $B \geq 1$ and $\mathcal{F}$ be a set of functions $f: \mathbb{R}^d \rightarrow [0,B]$. Let $Z_1, ..., Z_n$ be i.i.d. $\mathbb{R}^d$-valued random variables. For any $\alpha > 0$, $0 < \epsilon < 1$, and $ n \geq 1$, we have 
\begin{align*}
    P \left\{ \sup_{f \in \mathcal{F}} \frac{\frac{1}{n}\sum_{i=1}^n f(Z_i) - \mathbb{E}[f(Z)]}{\alpha + \frac{1}{n} \sum_{i=1}^n f(Z_i) + \mathbb{E}[f(Z)]} > \epsilon \right\} \leq 4 \mathbb{E} N(\frac{\alpha \epsilon}{5}, \mathcal{F}|Z_1^n, n^{-1} \| \cdot \|_1) \exp \left( \frac{-3 \epsilon^2 \alpha n}{40 B} \right).
\end{align*}
\label{lemma:sup_concentration}
\end{lem}

% \begin{lem}[\textbf{Bernstein's inequality} \citep{afol_lecture6}]
% Let $X_1, ..., X_n \sim X$ be i.i.d. real-valued random variables that satisfy the one-sided Bernstein's condition with parameter $b >0$. Then, for any $\epsilon > 0$ and $\delta \in [0,1]$, we have 
% \begin{align*}
%       P\left( \frac{1}{n} \sum_{i=1}^n X_i - \mathbb{E}X
%       < \frac{b}{n} \log(1/\delta) + \sqrt{ \frac{2(VarX) \log(1/\delta)}{n}  }
%       \right) \geq 1 - \delta.
% \end{align*}
% \end{lem}

\begin{lem}[\textit{Contraction property} \citep{afol_lecture2}]
Let $\phi: \mathbb{R} \rightarrow \mathbb{R}$ be a $L$-Lipschitz, then 
\begin{align*}
    \mathbb{E}_{\sigma} R_n \left( \phi \circ \mathcal{F} \right) \leq L \mathbb{E}_{\sigma} R_n \mathcal{F}.
\end{align*}
\end{lem}

\begin{lem}[{\citet[Lemma~1]{DBLP:journals/ijon/LeiDB16}}]
Let $\mathcal{F}$ be a function class and $P_n$ be the empirical measure supported on $X_1, ..., X_n \sim \mu$, then for any $r >0$ (which can be stochastic w.r.t $X_i$), we have 
\begin{align*}
    \mathbb{E}_{\sigma} R_n \{f \in \mathcal{F}: \| f \|_{n}^2 \leq r \} \leq \inf_{\epsilon > 0} \bigg[ \mathbb{E}_{\sigma} R_n \{ f \in \mathcal{F} - \mathcal{F}: \|f \|_{\mu} \leq \epsilon \}  + \sqrt{ \frac{2r \log N(\epsilon/2, \mathcal{F},\|\cdot \|_{n}) }{n}  } \bigg].
\end{align*}
\label{lemma:local_empirical_rademacher_bounded_by_covering_number_with_empirical_norm}
\end{lem}

\begin{lem}[{\citet[modification]{DBLP:journals/ijon/LeiDB16}}]
Let $X_1, ..., X_n$ be a sequence of samples and $P_n$ be the associated empirical measure. For any function class $\mathcal{F}$ and any monotone sequence $\{\xi_k\}_{k=0}^{\infty}$ decreasing to $0$, we have the following inequality for any non-negative integer $N$ 
\begin{align*}
    \mathbb{E}_{\sigma} R_n \{f \in \mathcal{F}: \|f \|_n \leq \xi_0 \} \leq 4 \sum_{k=1}^N \xi_{k-1} \sqrt{ \frac{\log \mathcal{N}(\xi_k, \mathcal{F}, \| \cdot \|_{n})}{n}  } + \xi_N.
\end{align*}
\label{lemma:refined_entropy_integral}
\end{lem}

\begin{lem}[\textit{Pollard's inequality}] 
Let $\mathcal{F}$ be a set of measurable functions $f: \mathcal{X} \rightarrow [0,K]$ and let $\epsilon >0, N$ arbitrary. If $\{X_i\}_{i=1}^N$ is an i.i.d. sequence of random variables taking values in $\mathcal{X}$, then 
\begin{align*}
    P \left( \sup_{f \in \mathcal{F}} \bigg| \frac{1}{N} \sum_{i=1}^N f(X_i) - \mathbb{E}[f(X_1)] \bigg| > \epsilon \right) \leq 8 \mathbb{E} \left[ N(\epsilon/8, \mathcal{F}|_{X_{1:N}}) \right] e^{\frac{-N \epsilon^2}{ 128 K^2}}.
\end{align*}

\end{lem}

% For completeness, we present here the definition of sub-root functions. 
% \begin{defn}[Definition 3.1 in \citep{bartlett2005}]
% A function $\psi: [0, \infty) \rightarrow [0, \infty)$ is \textit{sub-root} if it is non-negative, non-decreasing and $r \mapsto \psi(r) / \sqrt{r}$ is non-increasing for $r > 0$.
% \end{defn}

\begin{lem}[\textit{Properties of (bracketing) entropic numbers}]
Let $\epsilon \in (0, \infty)$. We have 
\begin{enumerate}
    \item $H(\epsilon, \mathcal{F}, \|\cdot \|) \leq H_{[]}(2 \epsilon, \mathcal{F}, \|\cdot \|)$;
    
    \item $H( \epsilon, \mathcal{F}|\{x_i\}_{i=1}^n, n^{-1/p} \cdot\| \cdot \|_p ) = H( \epsilon, \mathcal{F}, \| \cdot \|_{p,n} ) \leq H(\epsilon, \mathcal{F}| \{x_i\}_{i=1}^n, \| \cdot \|_{\infty}) \leq H(\epsilon, \mathcal{F}, \| \cdot \|_{\infty})$ for all $\{x_i\}_{i=1}^n \subset dom(\mathcal{F})$.
    
    \item $H(\epsilon, \mathcal{F} - \mathcal{F}, \| \cdot \|) \leq 2 H (\epsilon/2, \mathcal{F}, \| \cdot \|))$, 
    where $\mathcal{F} - \mathcal{F} := \{f - g: f, g \in \mathcal{F}\}$.
\end{enumerate}

\end{lem}

\begin{lem}[\textit{Entropic number of bounded Besov spaces} {\citep[Corollary~2.2]{Nickl2007BracketingME}}]
For $1 \leq p,q \leq \infty$ and $\alpha > d/p$, we have 
\begin{align*}
    H_{[]}(\epsilon, \bar{B}^{\alpha}_{p,q}(\mathcal{X}), \| \cdot \|_{\infty})  \lesssim \epsilon^{-d/\alpha}.
\end{align*}
\label{lemma:entropic_number_of_Besov}
\end{lem}

\begin{lem}[\textit{Approximation power of deep ReLU networks for Besov spaces} {\citep[a modified version]{suzuki2018adaptivity}}]
Let $1 \leq p,q \leq \infty$ and $\alpha \in (\frac{d}{p \land 2}, \infty)$. For sufficiently large $N \in \mathbb{N}$, there exists a neural network architecture $\Phi(L, m,S,B)$ with 
\begin{align*}
    L \asymp \log N, m \asymp N \log N, S \asymp N, \text{ and } B \asymp N^{d^{-1} + \nu^{-1}},
    \label{parameterize_net}
\end{align*}
where $\nu := \frac{\alpha - \delta}{2 \delta}$ and $\delta := d(p^{-1} - (1 + \floor{\alpha})^{-1})_{+}$
such that 
\begin{align*}
    \sup_{f_* \in \bar{B}^{\alpha}_{p,q}(\mathcal{X})} \inf_{f \in \Phi(L,W,S,B)} \|f - f_*\|_{\infty} \lesssim N^{-\alpha/d}.
\end{align*}
\label{lemma:approximation_power_for_Besov}
\end{lem}